%% file: paper.tex
\documentclass{article}
\usepackage[accepted]{icml2020}
\usepackage[utf8]{inputenc}
\usepackage[T1]{fontenc}
\usepackage{hyperref}
\usepackage{url}
\usepackage{amsfonts}
\usepackage{booktabs}
\usepackage{nicefrac}
\usepackage{microtype}
\usepackage{subcaption}
\usepackage{xcolor}
\usepackage{amsmath}
\usepackage{amsthm}
\usepackage{bm}
\usepackage{mathtools}
\usepackage{stmaryrd}
\usepackage{enumitem}
\usepackage{xspace}
\usepackage{inconsolata}
\usepackage{cancel}
\usepackage{soul}
\usepackage[frozencache]{minted}
\definecolor{bananamania}{rgb}{0.98, 0.91, 0.71}
\colorlet{codebg}{bananamania!30!white}
\sethlcolor{bananamania}

\usepackage{setspace}

\definecolor{mydarkblue}{rgb}{0,0.08,0.45}
  \hypersetup{ %
    pdftitle={LP-SparseMAP: Differentiable Relaxed Optimization for Sparse Structured Prediction},
    pdfauthor={Vlad Niculae and Andre F. T. Martins},
    pdfsubject={},
    pdfkeywords={},
    pdfborder=0 0 0,
    pdfpagemode=UseNone,
    colorlinks=true,
    linkcolor=mydarkblue,
    citecolor=mydarkblue,
    filecolor=mydarkblue,
    urlcolor=mydarkblue,
    pdfview=FitH}

\usepackage{tikz}
\usetikzlibrary{shapes,positioning,arrows}

\colorlet{hl}{purple}
\colorlet{treefactor}{teal!50!white}
\definecolor{budgetfactor}{rgb}{0.82, 0.1, 0.26}

\newcommand\Comment[1]{\hfill \textcolor{gray}{\#~\textit{#1}}}
\renewcommand{\algorithmicdo}{}
\renewcommand{\algorithmicthen}{}
\newcommand{\reals}{\mathbb{R}}
\newcommand{\simplex}{\triangle}
\newcommand{\p}{\bm{p}}  %
\newcommand{\pt}{\bm{\theta}}  %
\newcommand{\ppr}{\eta}
\newcommand{\pr}{\bm{\ppr}}  %
\newcommand{\mmg}{\mu}
\newcommand{\mg}{\bm{\mmg}}
\newcommand{\lbd}{\bm{\lambda}}

\newcommand{\DP}[2]{\langle#1,#2\rangle}
\newcommand{\pfrac}[2]{\frac{\partial #1}{\partial #2}}
\newcommand{\Mpo}{\mathcal{M}}
\newcommand{\Mb}{\bar{\bm{M}}}

\newcommand\setto{\stackrel{!}{=}}

\newcommand{\YY}{\mathcal{Y}}
\newcommand{\FF}{\mathcal{F}}
\renewcommand{\aa}{\bm{a}}
\newcommand{\mm}{\bm{m}}
\newcommand{\nn}{\bm{n}}
\renewcommand{\AA}{\bm{A}}
\newcommand{\MM}{\bm{M}}
\newcommand{\NN}{\bm{N}}
\newcommand{\CC}{\bm{C}}
\newcommand{\CCs}{\widetilde{\CC}}
\newcommand{\MMs}{\widetilde{\MM}}
\newcommand{\ZZ}{\bm{Z}}
\newcommand{\BB}{\bar{\bm{B}}}

\newcommand{\QQ}{\bm{Q}}
\newcommand{\JJ}{\bm{J}}

\newcommand{\MQM}{\Mb \QQ \Mb^\top}

\newcommand{\RR}{\mathbb{R}}

\newcommand{\smap}{SparseMAP\xspace}
\newcommand{\lpsmap}{LP-SparseMAP\xspace}
\newcommand{\adq}{AD\textsuperscript{3}\xspace}

\newcommand\hlf{\nicefrac{1}{2}~}

\newcommand*{\wrt}{\textit{w.\nobreak\hspace{.07em}r.\nobreak\hspace{.07em}t.}\@\xspace}
\newcommand*{\eg}{\textit{e.\nobreak\hspace{.07em}g.}\@\xspace}
\newcommand*{\ie}{\textit{i.\nobreak\hspace{.07em}e.}\@\xspace}
\newcommand*{\cf}{\textit{cf.}\@\xspace}

\DeclarePairedDelimiter{\iv}{\llbracket}{\rrbracket}

\DeclareMathOperator{\diag}{diag}
\DeclareMathOperator{\proj}{\mathbf{proj}}
\DeclareMathOperator{\bdiag}{bdiag}
\DeclareMathOperator{\conv}{conv}

\DeclareMathOperator{\map}{\mathsf{MAP}}
\DeclareMathOperator{\marg}{\mathsf{Marginals}}
\DeclareMathOperator{\softmax}{\mathsf{softmax}}
\DeclareMathOperator*{\argmax}{arg\,max}
\DeclareMathOperator*{\argmin}{arg\,min}
\DeclareMathOperator{\clip}{clip}

\newcommand{\agr}[3]{#1_{(#2)} = #3}

\newcommand\eqnref[1]{Eq.\,\,\ref{#1}}

\newcommand\suppref[1]{App.\,\,\ref{supp:#1}}

\newtheorem{proposition}{Proposition}
\newtheorem{corollary}{Corollary}[proposition]
\newtheorem{lemma}{Lemma}

\icmltitlerunning{LP-SparseMAP}

\begin{document}

\twocolumn[
\icmltitle{LP-SparseMAP: Differentiable Relaxed Optimization \\for Sparse
Structured Prediction}

\begin{icmlauthorlist}
\icmlauthor{Vlad Niculae}{it}
\icmlauthor{Andr\'e F.~T.~Martins}{it,un,ist}
\end{icmlauthorlist}
\icmlaffiliation{it}{Instituto de Telecomunica\c{c}\~oes, Lisbon, Portugal}
\icmlaffiliation{ist}{Instituto Superior T\'ecnico, University of Lisbon,
Portugal}
\icmlaffiliation{un}{Unbabel, Lisbon, Portugal}
\icmlcorrespondingauthor{Vlad Niculae}{vlad@vene.ro}
\icmlcorrespondingauthor{Andr\'e
F.~T.~Martins}{andre.t.martins@tecnico.ulisboa.pt}
\icmlkeywords{Structured Prediction, Differentiable Optimization}

\vskip 0.3in
]

\printAffiliationsAndNotice{}

\begin{abstract}
Structured predictors require solving a combinatorial optimization problem over
a large number of structures, such as dependency trees or alignments.
When embedded as
structured hidden layers in a neural net, argmin differentiation and
efficient gradient computation are further required.
Recently, SparseMAP has been proposed as a differentiable,
sparse alternative to maximum a posteriori (MAP) and marginal inference.
SparseMAP returns an interpretable combination of a small number of structures;
its sparsity being the key to efficient optimization.
However, SparseMAP requires access to an exact MAP oracle in the structured
model, excluding, \eg, loopy graphical models or logic constraints, which generally require approximate inference.
In this paper, we introduce \emph{LP-SparseMAP}, an extension of SparseMAP
addressing this limitation via a {local polytope} relaxation.
LP-SparseMAP uses the flexible and powerful {%domain specific
language} of {factor graphs} to define
expressive hidden structures, supporting coarse decompositions,
hard logic constraints, and higher-order correlations.
We derive the forward and backward algorithms needed for using LP-SparseMAP as a
structured hidden or output layer. Experiments in three structured tasks show
benefits versus SparseMAP and Structured SVM.
\end{abstract}

\section{Introduction}
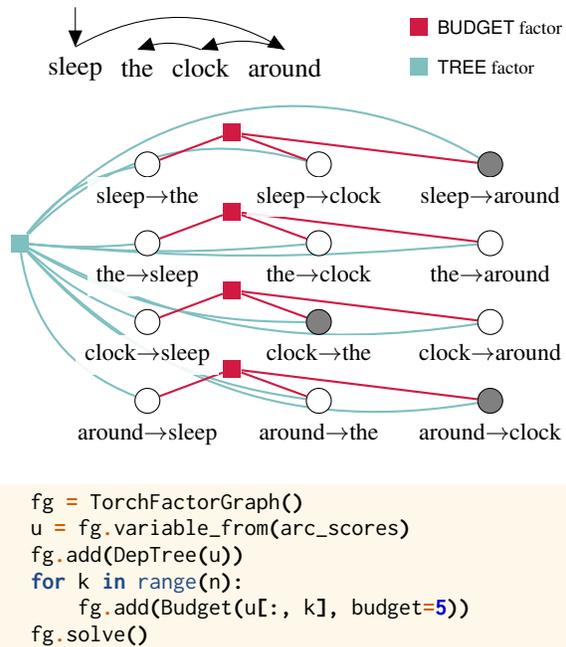
\begin{figure}[t]
\centering
\input{fg.tex}
\begin{minted}[bgcolor=codebg,fontsize=\footnotesize]{python}
   fg = TorchFactorGraph()
   u = fg.variable_from(arc_scores)
   fg.add(DepTree(u))
   for k in range(n):
       fg.add(Budget(u[:, k], budget=5))
   fg.solve()
\end{minted}
\caption{\label{fig:valency}Parsing model with valency constraints: each
``head'' word is constrained to have at most $k$ ``modifiers''.
{\lpsmap} is the first method for tractable, differentiable decoding in such a
model. Below: abridged implementation using our library
(more in \suppref{code}).}
\end{figure}
The data processed by machine learning systems often has underlying structure:
for instance, language data has inter-word dependency trees, or alignments,
while image data can reveal object segments.  As downstream models
benefit from the hidden structure, practitioners typically resort to
\emph{pipelines}, training a structure predictor on labelled
data, and using its output as features.
This approach requires annotation, suffers from error propagation, and cannot
allow the structure predictor to adapt to the downstream task.

Instead, a promising direction is to treat structure as latent, or
\textbf{hidden}: learning a structure predictor without supervision,
together with the downstream model in an end-to-end fashion.
Several recent approaches were proposed to tackle this, based on differentiating
through marginal inference \citep{kim-structuredattn,lapata}, noisy gradient
estimates \citep{spigot,rlspinn},
or both \citep{caio-iclr,caio-acl}.
The work in this area
requires specialized, structure-specific algorithms either for
computing gradients or for sampling, limiting the choice of the practitioner
to a catalogue of supported types structure.
A slightly more general approach is \textbf{\smap} \citep{sparsemap}, which
is differentiable and outputs combinations of a small number of structures, requiring only an
algorithm for finding the highest-scoring structure (maximum a posteriori, or
MAP).
When increased expressivity is required, for instance through logic constraints
or higher-order interactions, the search space becomes much more complicated,
and MAP is typically intractable.
For example, adding constraints on the depth of a parse tree makes the problems NP-hard.
Our work improves the hidden structure modeling freedom available to
practitioners, as follows.
\begin{itemize}[nolistsep,itemsep=0.5em]
\item We propose a generic method for differentiable structured hidden layers,
based on the flexible domain-specific language of \textbf{factor graphs},
familiar to many structured prediction practitioners.

\item We derive an efficient and globally-convergent ADMM algorithm for the
forward pass.

\item We prove a compact, efficient form for the backward pass,
reusing quantities precomputed in the forward pass,
without having to unroll a computation graph.

\item Our overall method is \textbf{modular}: new factor types can be added to
our toolkit just by providing a MAP oracle, invoked by the generic forward and
backward pass.

\item
The generic approach may be overridden when specialized algorithms are
available. We derive efficient specialized algorithms
for core building block factors such as pairwise, logical \textsf{OR}, negation, budget constraints, \emph{etc}., ensuring our
toolkit is expressive \emph{out-of-the-box}.
\end{itemize}

We show empirical improvements on inducing latent trees
on arithmetic expressions,
bidirectional alignments in natural language inference, and multilabel
classification.
Our library, with C++ and python frontends,
is available at \url{https://github.com/deep-spin/lp-sparsemap}.

\section{Background}
\subsection{Notation}
We denote scalars, vectors and matrices as $a$, $\bm{a}$, and $\bm{A}$,
respectively.
The set of indices $\{1, \dots, d\}$ is denoted $[d]$.
The Iverson bracket $\iv{C}$ takes the value $1$ if the condition $C$ is true,
otherwise $0$.
The indicator vector $\bm{e}_i$ is defined as $[\bm{e}_i]_k \coloneqq
\iv{i=k}$.
The $i$\textsuperscript{th} column of matrix $\bm{A}$ is
$\bm{a}_i$.
The canonical simplex is
$\simplex \coloneqq \{ \p \in \reals^d~:~\DP{\bm{1}}{\p} = 1, \bm{p} \geq \bm{0}\},$
and the convex hull is
$\conv\{ \bm{a}_1, \dots, \bm{a}_d\} \coloneqq \{ \bm{A}\p : \p \in \simplex\}$.
We denote row-wise stacking of $\bm{A}_i \in \reals^{m_i
\times d}$ as
$[\bm{A}_1, \dots,  \bm{A}_k] \in \reals^{\left(\sum_i m_i\right) \times d}$.
Particularly, $[\bm{a},
\bm{b}]$ is the concatenation of two (column) vectors.
Given a vector $\bm{b} \in \reals^d$, $\diag(\bm{b}) \in \reals^{d \times d}$
is the diagonal matrix with $\bm{b}$ along the diagonal.
Given matrices $\bm{B}_1, \dots, \bm{B}_k$
of arbitrary dimensions $\bm{B}_i \in \reals^{m_i \times n_i}$,
define the block-diagonal matrix\\
{\renewcommand*{\arraycolsep}{1pt}%
$\bdiag(\bm{B}_1, \dots, \bm{B}_k) =
{\scriptsize
\begin{bmatrix}
        \bm{B}_{1} & \cdots & \bm{0}\\
        \vdots & \ddots & \vdots \\
        \bm{0} & \cdots & \bm{B}_{k} \\
    \end{bmatrix}}
\in \reals^{\sum_i\!m_i \times \sum_i\!n_i}$.}

\subsection{Tractable structured problems}

Structured prediction involves searching for valid structures over a large,
combinatorial space $y \in \YY$.  We assign a vector representation $\aa_y$ to
each structure. For instance, we may consider structures to be joint assignments
of $d$ binary variables (corresponding to parts of the structure)
and define $(\aa_y)_i = 1$ if variable $i$ is turned on
in structure $y$, else $0$. The set of \emph{valid structures} $\mathcal{Y}$ is
typically non-trivial. For example, in \emph{matching} problems between $n$
workers and $n$ tasks, we have $d=n^2$ binary variables, but the only legal
assignments give exactly one task to each worker, and one worker to each task.

\paragraph{Maximization (MAP).}
Given a score vector over parts $\pr$,
we assign a score $\pt_y = \DP{\aa_y}{\pr}$ to each structure.
Assembling all $\aa_y$ as columns of a matrix $\AA$,
the highest-scoring structure is the one maximizing
\begin{equation}\label{eqn:map}
\max_{y \in \YY} \DP{\pr}{\aa_y}
= \max_{\p \in \simplex} \DP{\pr}{\AA\p}.
\end{equation}

$\Mpo_{\AA} = \conv \{ \aa_y : y \in \YY \}$ is called the marginal
polytope
\citep{wainwright},
and points $\mg \in \Mpo_{\AA}$
are expectations
$\mathbb{E}_{y \sim \p}[\aa_y]$ under some
$\p \in \simplex$.

In the sequel, we split $\AA = [\MM,\NN]$ such that $\mg = \MM \p$ is the output
of interest, (\eg, variable assignments), while $\NN \p$ captures additional
structures or interactions (\eg, transitions in sequence tagging).
We denote the corresponding division of the score vector as $\pr = [\pr_M,
\pr_N]$.
This distinction is not essential, as we may always take $\MM = \AA$
and $\NN=[]$
(\ie, treat additional interactions as first-class variables),
but it is more consistent with pairwise Markov Random Fields (MRF).

\textbf{Examples.} A model with $3$ variables and an \textsf{XOR}
constraint (exactly one variable may be on) has possible configurations
$\mm_y = \bm{e}_y$ for $y \in \{1, 2, 3\}$, thus $\MM=\bm{I}$, and no
additionals ($\NN=[]$).
A model with the same dimension but without the constraint has all $2^3$ possible
configurations as columns of $\MM$, still with no additionals.
One such configuration is $y=\texttt{011}$, with $\mm_{y} = [0, 1, 1]$.
(Throughout this paper, $y$ is an arbitrary index type with no mathematical
properties; we may as well use an integer base-2 encoding.)
A \emph{sequence} model with no constraints will have the same valid
configurations, but will include additionals for transition potentials: here it
is sufficient to have an additional bit for each consecutive pair of
variables, assigning $1$ if both variables are simultaneously active. For
$y=\texttt{011}$ this gives $\nn_{y} = [0, 1]$.

\paragraph{Optimization as a hidden layer.}
Hidden layers in a neural network are vector-to-vector mappings, and learning is
typically done using stochastic gradients. We may cast structured maximization
in this framework.
Assuming fixed tie-breaking, we may regard the MAP computation as a function
that takes the scores $\pr$ and outputs a vector of variable assignments $\mg
\in [0, 1]^d$,
\begin{equation}
\begin{gathered}
\map_{\AA}(\pr) \coloneqq \mg\\
\quad\text{where}%
\quad\mg\coloneqq \mm_y,\quad y = \argmax_{y \in \YY} \DP{\pr}{\aa_y}.
\end{gathered}
\end{equation}
The solution is always a vertex in $\{0, 1\}^d$, and, for almost all $\pr$,
small changes to $\pr$ do not change what the highest-scoring structure is.
Thus, wherever $\map_{\AA}$ is continuous, its gradients are null,
rendering it unsuitable as a hidden layer in a neural
network trained with gradient-based optimization \citep{spigot}.

\paragraph{Marginal inference.}
In unstructured models (\eg, attention mechanisms),
discrete maximization has the same null gradient issue identified in the
previous paragraph, thus it is commonly replaced by its relaxation $\softmax(\bm{x})$.
Denote the Shannon entropy of a distribution $\p\in\simplex$ by $H(\p) \coloneqq -\sum_j p_j \log p_j$.
The structured equivalent of softmax
is the entropy-regularized problem
\begin{equation}
\max_{\p \in \simplex} \DP{\pr}{\AA\p} + H(\p),
\end{equation}
whose solution is %
$p_y^\star \propto \exp \DP{\aa_y}{\pr}.$
This \emph{Gibbs distribution}
is dense and induces a marginal
distribution over variable assignments \citep{wainwright}:
\begin{equation}
   \marg_{\AA}(\pr) \coloneqq \mg \quad\text{where}\quad\mg \coloneqq
\mathbb{E}_{\p^\star}[\mm_y].
\end{equation}
While generally intractable,
for certain models, such as sequence tagging, one can efficiently
compute $\marg_{\AA}(\pr)$ and $\nabla \marg_{\AA}(\pr)$
\citep[often, with dynamic programming,][]{kim-structuredattn}.
In many, it is intractable, \eg, matching
\citep[Section
3.5]{valiant1979complexity,taskar-thesis},
dependency parsing with valency constraints
\citep{McDonald2007}.

\paragraph{SparseMAP} \citep{sparsemap} is a differentiable
middle ground between maximization and expectation.
It is defined via the quadratic objective
\begin{equation}\label{eqn:smap}
    \max_{\p \in \simplex} \DP{\pr}{\AA\p} - \frac{1}{2} \|\MM\p \|^2.
\end{equation}
where an optimal sparse distribution $\p$ and the unique $\mg = \MM\p$
can be efficiently computed via the \emph{active set} method
\citep[Ch.~16.4 \& 16.5]{nocedalwright},
a generalization of Wolfe's \emph{min-norm point method} \citep{mnp}
and an instance of \emph{conditional gradient}
\citep{fw}.
Remarkably, the active set method only requires
calls to a maximization oracle (\ie, finding the
highest-scoring structure repeatedly, after adjustments),
and has linear, finite convergence.
Thus, {\smap} can be computed efficiently even
when marginal inference is not available, potentially
turning any structured problem with a maximization
algorithm available into a differentiable sparse structured hidden layer.
The sparsity not only brings computational advantages, but also aids
visualization and interpretation.

However, the requirement of an exact maximization algorithm is still a rather
stringent limitation. In the remainder of the section, we look into a flexible family of
structured models where maximization is hard. Then, we extend {\smap} to cover
all such models.

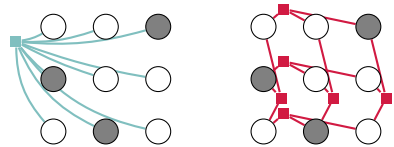
\begin{figure}[t]
\centering
\input{matchfg.tex}
\caption{\label{fig:match}Matching model under two equivalent decompositions.
Left: a coarse one with a single factor. Right: a fine
one with multiple \textsf{XOR} factors.}
\end{figure}

\subsection{Intractable structured problems\\and factor graph representations}

We now turn to more complicated structured problems, consisting of multiple
interacting subproblems. As we shall see, this covers many interesting problems.

Essentially, we represent the global structure as assignments to $d$ variables,
and posit a \textbf{decomposition} of the problem into local
\emph{factors} $f \in \mathcal{F}$, each encoding locally-tractable scoring and
constraints \citep{fg}.
A factor may be seen as smaller structured
subproblem.
Crucially, factor must agree whenever they \textbf{overlap},
rendering the subproblems interdependent, non-separable.

\paragraph{Examples.}
Figure~\ref{fig:valency} shows a factor graph for
a dependency parsing problem in which prior
knowledge dictates \emph{valency constraints}, \ie, disallowing words to be
assigned more than $k$ dependent modifiers. This encourages depth, preventing trees
from being too flat.
For a sentence with $m$ words, we use $m^2$ binary variables for every possible
arc, (including the root arcs, omitted in the figure). The global tree
factor disallows assignments that are not trees, and the $m$
\emph{budget} constraint factors, each governing $m-1$ different variables,
disallow more than $k$ dependency arcs out of each word.
Factor graph representations are often \textbf{not unique}. For instance,
consider a matching (linear assignment) model (Figure~\ref{fig:match}).
We may employ a
coarse factorization consisting of a single \emph{matching} factor,  for which
maximization is
tractable thanks to the Kuhn-Munkres algorithm \citep{km}.
This problem can also be represented using multiple \textsf{XOR} factors,
constraining that each row and each column must have exactly (exclusively) one
selected variable.

Denote the variable assignments as $\mg \in [0, 1]^d$.
We regard each factor $f$ as a separate structured model in its own right,
encoding its permissible assignments as columns of a matrix
$\AA_f = [\MM_f, \NN_f]$, and define a \emph{selector matrix}
$\CC_f$ such that $\CC_f \mg$ ``selects'' the variables from the global vector $\mg$ covered by
the factor $f$. Then, a \emph{valid} global assignment can be represented
as a tuple of local assignments $y_f$, provided that the
agreement constraints are satisfied:
\begin{equation}
    \YY = \{y = (y_f)|_{f\in\FF}: \exists~\mg,~
    \forall f\in\FF,~\CC_f \mg = \mm_{y_f}\}.
\end{equation}
Finding the highest scoring structure has the same form as in the tractable case,
but the discrete agreement constraints in $\YY$ make it difficult to compute,
even when each factor is computationally friendly:
\begin{equation}\label{eqn:graphmap}
    \max_{y \in \YY} \sum_{f\in\FF} \DP{\pr_f}{\aa_{y_f}}.
\end{equation}
In the tractable case, we were able to relax the discrete maximization into
a continuous one with respect to a distribution over global configurations $\p
\in \simplex$ (\eqnref{eqn:map}). We take the same approach, but \emph{locally}, considering
distributions over \emph{local} configurations $\p_f \in \simplex_f$ for each
factor.
For compactness, we shall use the concatenations
\begin{equation*}
\p \coloneqq [\p_{f_1}, \dots, \p_{f_n}],\quad
\CC \coloneqq [\CC_{f_1}, \dots, \CC_{f_n}]
\end{equation*}
and the block-diagonal matrices
\begin{equation*}
\AA \coloneqq \bdiag(\AA_{f_1}, ..., \AA_{f_n}),
\MM \coloneqq \bdiag(\MM_{f_1}, ..., \MM_{f_n}).
\end{equation*}
We may then write the optimization problem
\begin{equation}\label{eqn:lpmap}
\begin{aligned}
\underset{\mg,~\p}{\text{maximize}}\quad&
\sum_{f \in \FF} \DP{\pr_f}{\AA_f \p_f} \\
\quad\text{subject to}\quad&
\p \in \simplex_{f_1} \times \simplex_{f_2} \times \dots \times \simplex_{f_n},
\\&\CC \mg = \MM \p,
\end{aligned}
\end{equation}
continuously relaxing each factor independently
while enforcing agreement.
The objective in \eqnref{eqn:lpmap} is separable, but the constraints are not.
The feasible set,
\begin{equation}
\mathcal{L} = \{ \AA \p :
\p \in \simplex_{f_1} \times
\dots \times \simplex_{f_n},
~ \CC \mg = \MM \p
\},
\end{equation}
is called the \emph{local polytope} and satisfies %
$\mathcal{L} \supseteq \mathcal{M} = \conv \{ \aa_y : y \in \YY \}$.
Therefore, \eqref{eqn:lpmap} is a relaxation of \eqref{eqn:graphmap},
known as LP-MAP
\citep{wainwright}.
In general, the inclusion $\mathcal{L} \supseteq
\mathcal{M}$ is strict.
Many LP-MAP algorithms exploiting the graphical model
structure have been proposed, from the perspective of
message passing or dual decomposition
\citep{trws,trws-kolmogorov,komodakis,globerson,koo}.
In particular, \adq \citep{ad3} tackles LP-MAP
by solving a \smap-like quadratic subproblem for each factor.

It may be tempting to consider building a differentiable structured hidden layer
by using SparseMAP with an LP-MAP approximate oracle. However, since LP-MAP is
an outer relaxation, solutions are in general not feasible, leading to
divergence.
Instead, in the sequel, we apply the LP relaxation to a smoothed objective,
resulting in a general algorithm for sparse differentiable inference.

\section{LP-SparseMAP}
By analogy to \eqnref{eqn:smap}, we propose the differentiable
LP-SparseMAP inference strategy:
\begin{equation}\label{eqn:lpsmap}
\begin{aligned}
\underset{\mg,~\p}{\text{maximize}}\quad&\Big(
\sum_{f \in \FF} \DP{\pr_f}{\AA_f \p_f}\Big) - \hlf \| \mg \|^2\\
\quad\text{subject to}\quad&
\p \in \simplex_{f_1} \times \simplex_{f_2} \times \dots \times \simplex_{f_n},
\\&\CC \mg = \MM \p.
\end{aligned}
\end{equation}
Unlike LP-MAP (\eqnref{eqn:lpmap}), \lpsmap has a
non-separable $\ell_2$ term in the objective.
The next result reformulated the problem as separable
consensus optimization.
\begin{proposition}\label{prop:separable}
Denote by $\deg(j) = |\{f \in \mathcal{F}: j \in f \} | > 0$,
the number of factors governing $\mu_j$.%
\footnote{Variables not attached to any factor can be removed from
the problem, so we may assume $\deg(j) > 0$.}
Define $\bm{\delta}$ as
$\delta_j = \sqrt{\deg(j)}$,
and $\bm{D} = \diag(\CC\bm{\delta})$.
Denote $\CCs = \bm{D}^{-1} \CC, \MMs = \bm{D}^{-1} \MM$.
Then, the problem below is equivalent to \eqref{eqn:lpsmap}:
\begin{equation}
\label{eqn:lpsmapsep}
\begin{aligned}
\underset{\mg,~\p}{\text{maximize}}\quad&
\sum_{f\in\FF} \Big(\DP{\pr_f}{\AA_f \p_f} - \hlf\| \MMs_f \p_f \|^2\Big)
\\\text{subject to}\quad&
\p \in \simplex_{f_1} \times \simplex_{f_2} \times \dots \times \simplex_{f_n},
\\&\CCs \mg = \MMs \p.
\end{aligned}
\end{equation}
\end{proposition}

\begin{proof}
The constraints $\CC\mg = \MM\p$ and $\CCs\mg = \MMs\p$ are equivalent since
$\bm{\delta} > 0$ ensures $\bm{D}$ invertible.
It remains to show that, at feasibility, $\|\mg\|^2 = \|\MMs\p\|^2$.
This follows from $\|\mg\|^2 = \|\CCs\mg\|^2$ (shown in \suppref{lpsmap}).
\end{proof}

\subsection{Forward pass}
Using this reformulation, we are now ready to introduce an ADMM algorithm
\citep{glowinski,gabay,admm}
for maximizing \eqnref{eqn:lpsmapsep}.
The algorithm is given in Algorithm~\ref{alg:ad3qp}
and derived in \suppref{ad3qp}. Like \adq, it
iterates alternating between:
\begin{enumerate}[nolistsep,itemsep=0.5em]
\item solving a \smap subproblem for each factor;
(With the active set algorithm, this requires only cheap calls to a MAP oracle.)
\item enforcing global agreement by averaging;
\item performing a gradient update on the dual variables.
\end{enumerate}
\begin{proposition}Algorithm~\ref{alg:ad3qp} converges to a solution of
\eqref{eqn:lpsmap}; moreover, the number of iterations needed to reach $\epsilon$
dual suboptimality is $\mathcal{O}(\nicefrac{1}{\epsilon})$.
\end{proposition}
\begin{proof}
The algorithm is an instantiation of ADMM to
\eqnref{eqn:lpsmapsep}, inheriting the proof of convergence of ADMM.
\citep[Appendix A]{admm}.
From Proposition~\ref{prop:separable}, this problem is
equivalent to \eqref{eqn:lpsmap}.
Finally, the rate of convergence is established by
\citet[Proposition 8]{ad3}, as the problems
differ only through an additional regularization term in the objective.
\end{proof}
\input{narralgo_fwd.tex}

When there is a single factor, \ie, $\mathcal{F}=\{f\}$,
running for one iteration with $\gamma=0$ recovers {\smap}.
In practice, in the inner active set solver we use warm starts and perform
a small number of MAP calls. This leads to an algorithm more similar in spirit
to Frank-Wolfe splitting \citep{pmlr-v84-gidel18a}, with the key difference that
by solving the nested QPs we obtain the necessary quantities to ensure a more
efficient backward pass, as described in the next section.

\input{narralgo_bck.tex}

\subsection{Backward pass}
Unlike marginal inference, \lpsmap encourages the local distribution at each
factor to become sparse,
and yields a simple form for the \lpsmap Jacobian,
defined in terms of the local \smap Jacobians of each
factor (\suppref{smap_grad}).
Denote the local solutions $\mg_f = \MMs \p_f$ and the Jacobians of the SparseMAP
subproblem for each factor as
\begin{equation}
\JJ_{f,M} \coloneqq \pfrac{\mg_f}{\pr_{f, M}},\quad
\JJ_{f,N} \coloneqq \pfrac{\mg_f}{\pr_{f, N}}.
\end{equation}
When using the active set
algorithm for SparseMAP, $\JJ_{f,\{M,N\}}$ are
precomputed in the forward pass \citep{sparsemap}.
The \lpsmap backward pass combines the local Jacobians while taking into account
the agreement constraints, as shown next.
\begin{proposition}
Let
$\JJ_M = \bdiag(\JJ_{f,M})$
and
$\JJ_N = \bdiag(\JJ_{f,N})$
denote the block-diagonal matrices of local SparseMAP Jacobians.
Let $\JJ=\JJ^\top \in \reals^{d \times d}$ satisfying
\begin{equation}\label{eq:fixed_point_eig}
\JJ \coloneqq
\bm{\CCs^\top \JJ_M \CCs}~\JJ.
\end{equation}
\begin{equation}
\text{Then,}\quad
\pfrac{\mg}{\pr_M} = \JJ
\quad\text{and}\quad
\pfrac{\mg}{\pr_N} = \JJ \CCs^\top\JJ_N.
\end{equation}
\end{proposition}
The proof is given in \suppref{lpsmap_grad}, and $\bm{J}$ may be
computed using an eigensolver. However,
to use LP-SparseMAP as a hidden layer, we don't need a materialized
Jacobian, just its multiplication by
an arbitary vector $\bm{d}\in\reals^d$, \ie,
%access to Jacobian-vector products
\[
\Big(\pfrac{\mg}{\pr_M}\Big)^\top \bm{d}, \qquad\text{and}\qquad
\Big(\pfrac{\mg}{\pr_N}\Big)^\top \bm{d}.
\]
These can be computed iteratively by Algorithm~\ref{alg:ad3qp_backward}.
Since $\CC_f$ are highly sparse and structured selector matrices,
lines 5  %%VN: broken refs
%\ref{line:backsplit}
and 8
%\ref{line:backmerge}
are fast indexing
operations followed by scaling;
the bulk of the computation is line 6,
%\ref{line:backjacob},
which can be seen as
\textbf{invoking the backward pass of each factor}, as if that factor were
alone in the graph. The structure of
Algorithm~\ref{alg:ad3qp_backward} is similar to Algorithm~\ref{alg:ad3qp},
however, our backward is much more efficient than ``unrolling'' Algorithm~\ref{alg:ad3qp}
within a computation graph:
Our algorithm only requires access to the final state of the ADMM solver
(Algorithm~\ref{alg:ad3qp}), rather
than all intermediate states, as would be required for unrolling.

\subsection{Implementation and specializations}
The forward and backward passes of {\lpsmap}, described above,
are appealing from the perspective of modular implementation.
The outer loop interacts with a factor with only two interfaces:
a \texttt{SolveSparseMAP} function and a \texttt{JacobianTimesVector} function.
In turn, both methods can be implemented in terms of a \texttt{SolveMAP}
maximization oracle \citep{sparsemap}.

For certain factors, such as the logic constraints in Table~\ref{tab:logic},
faster direct implementations of
\texttt{SolveSparseMAP} and \texttt{JacobianTimesVector} are available,
and our algorithm easily allows specialization. This is appealing from a testing
perspective, as the specializations must agree with the generic implementation.
For example, the exclusive-or \textsf{XOR} factor requires that
exactly one out of $d$ variables can be on. Its marginal polytope is the
convex hull  of allowed assignments,
$\Mpo_\text{XOR} = \conv \{
\bm{e}_1, \dots, \bm{e}_d\} = \simplex^d$.
The required \smap subproblem with degree corrections is
\begin{equation}\label{eqn:factor_xor_body}
\begin{aligned}
\mathrm{minimize}~& \hlf \| \mg - \pr \|^2_2\\
\mathrm{subject\,to}~&
\sum_{j=1}^d \delta_j \mu_j = 1,~\text{and}~
0 \leq \mu_i \leq \nicefrac{1}{\delta_i}.
\end{aligned}
\end{equation}
When $\bm{\delta} = \bm{1}$ this is a projection onto the
simplex (sparsemax), for which efficient algorithms are well-studied \citep{sparsemax}.
For general $\bm{\delta}$, the algorithm of \citet{Pardalos1990} applies, and the
backward pass involves a generalization of the sparsemax Jacobian.

In \suppref{specialized}, we derive specialized forward
and backward passes for \textsf{XOR}, and the constraint factors in
Table~\ref{tab:logic}, as well as
for negated variables,
\textsf{OR}, \textsf{OR-Output}, \textsf{Knapsack} and pairwise (Ising) factors.

\begin{table}
\caption{Examples of logic constraint factors.\label{tab:logic}}
\centering
\begin{tabular}{l l}
\toprule
name & constraints \\
\midrule
\textsf{XOR} {\footnotesize (exactly one)} & $\sum_{i=1}^d \mmg_i = 1$ \\
\textsf{AtMostOne} & $\sum_{i=1}^d \mmg_i \leq 1$ \\
\textsf{OR} & $\sum_{i=1}^d \mmg_i \geq 1$ \\
\textsf{BUDGET} & $\sum_{i=1}^{d} \mmg_i \leq B$ \\
\textsf{Knapsack} & $\sum_{i=1}^d c_i \mmg_i \leq B$ \\
\textsf{OROut} & $\sum_{i=1}^{d-1} \mmg_i \geq \mmg_d; \mmg_i \leq \mmg_d$ for
all $i$\\
\bottomrule
\end{tabular}
\end{table}

\section{LP-SparseMAP loss for structured outputs}
So far, we described {\lpsmap} for structured
hidden layers.  When supervision is available, either as a downstream objective
or as partial supervision, a natural convex loss
relaxes the SparseMAP loss \citep{sparsemap}:
\begin{equation}
\ell(\pr, y)\!\coloneqq\!%
\max_{\p, \mg} \sum_f \DP{\AA_f^\top \pr_f}{\p_f - \bm{e}_{y_f}}
+ \frac{1}{2}(\| \bm{m}_y \|^2 - \| \mg \|^2),
\end{equation}
under the constraints of \eqnref{eqn:lpsmap}.
Like the \smap loss, this \lpsmap loss
falls into the
recently-proposed class of
Fenchel-Young losses \citep{fylosses},
which confirms its convenient properties, notably the
\emph{margin} property \citep[Proposition~8]{fylossesjmlr}.
Its gradients are obtained from the {\lpsmap} solution $(\mg, \p)$ as
\begin{gather}
\nabla_{\pr_M} \ell(\pr, y) = \mg - \bm{m}_y, \\
\nabla_{\pr_f,N} \ell(\pr, y) = \NN_f \p_f - \bm{n}_{y_f}.
\end{gather}
When already using LP-SparseMAP as a hidden layer, this loss provides a natural
way to incorporate supervision on the latent structure at no additional cost.

\section{Experiments}
In this section, we demonstrate {\lpsmap} for learning complex latent structures
on both toy and real-world datasets, as well as on a structured output task.
Learning hidden structures solely from a downstream objective is challenging
for powerful models that can bypass the latent component entirely. For this
reason, we design our experiments using
simpler, smaller networks where the inferred structure is an un-bypassable
\emph{bottleneck}, ensuring the predictions depend on
it. We use Dynet \citep{dynet} and list hyperparameter configurations
and ranges in \suppref{experimental}.

\subsection{ListOps valency tagging}
\begin{figure}
\centering
\includegraphics[width=.48\textwidth]{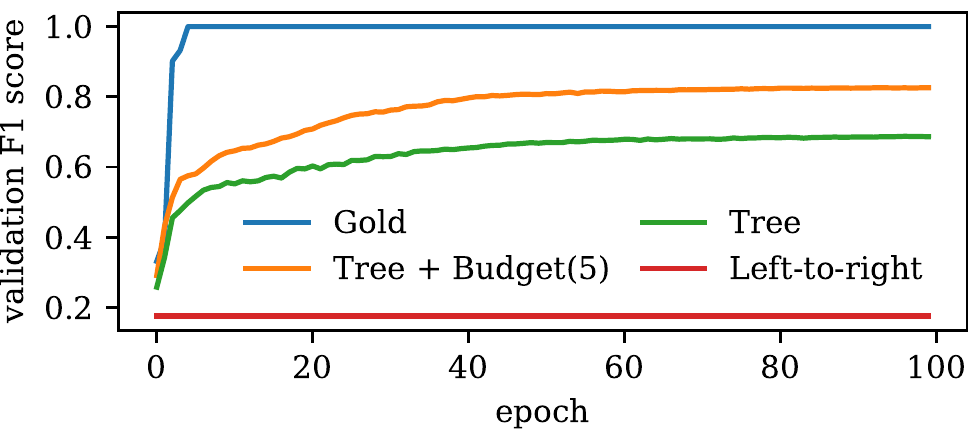}
\caption{$F_1$ score for tagging ListOps nodes with their valency,
using a latent tree. Incorporating inductive bias via budget
constraints improves performance.\label{fig:listops}}
\end{figure}

The ListOps dataset \citep{listops} is a synthetic collection of bracketed
expressions, such as \hl{\texttt{[max 2 9 [min 4 7 ] 0 ]}}. The arguments are lists of integers, and the operators are
set summarizers such as \texttt{median}, \texttt{max}, \texttt{sum}, etc. It was proposed as a litmus
test for studying latent tree learning models, since the syntax is essential to
the semantics. Instead of tackling the challenging task of learning to
\emph{evaluate} the expressions, we follow \citet{caio-acl} and study a
\emph{tagging} task: labeling each operator with the number of arguments it
governs.

\paragraph{Model architecture.} %
We encode the sequence with a BiLSTM, yielding vectors $\bm{h}_1, \dots,
\bm{h}_L$. We compute the score of dependency arc $i
\rightarrow j$ as the dot product between the outputs of two mappings,
one for encoding the head and one for the modifier (target word):
\begin{equation*}
\begin{gathered}
\bm{f}_\text{hd}(\bm{h}) = \bm{W}_\text{hd} \bm{h} + \bm{b}_\text{hd};\quad
\bm{f}_\text{mo}(\bm{h}) = \bm{W}_\text{mo} \bm{h} + \bm{b}_\text{mo};
\\[.5\baselineskip]
\eta_{i \rightarrow j} =
\DP{\bm{f}_\text{hd}(\bm{h}_i)}{\operatorname{ReLU}(\bm{f}_\text{mo}(\bm{h}_j))}.
\end{gathered}
\end{equation*}
We perform {\lpsmap} optimization to get the sparse arc posterior probabilities,
using different factor graph structures $\mathcal{F}$, described in the next
paragraph.
\begin{equation}
\mg = \operatorname{LP-SparseMAP}_{\mathcal{F}}(\pr)
\end{equation}
The arc posteriors $\mg$ correspond to a sparse combination of dependency trees.
We perform one iteration of a Graph Convolutional Network (GCN) along the edges
in $\mg$. Crucially, the input to
the GCN is not the BiLSTM output $(\bm{h}_1, \dots, \bm{h}_L)$ but a
``de-lexicalized'' sequence $(\bm{v}, \dots, \bm{v})$ where $\bm{v}$ is a
learned parameter vector, repeated $L$ times regardless of the tokens.
This forces the predictions to rely on the GCN and
thus on the latent trees, preventing the model from using the global BiLSTM to
``cheat''.
The GCN produces contextualized representations $(\bm{g}_1, \dots, \bm{g}_L)$
which we then pass through an output layer to predict the valency label for each
operator node. %

\paragraph{Factor graphs.} Unlike \citet{caio-acl}, who use projective
dependency parsing, we consider the general non-projective case,
making the problem more challenging.
The MAP oracle is the maximum arborescence algorithm
\citep{Chu1965,Edmonds1967}.
\begin{table}[t]
\small%
\centering%
\caption{ListOps tagging results with non-projective latent trees. The budget
constraints bring improvement.\label{tab:listops}}
\begin{tabular}{l r r r r}
\toprule
& \multicolumn{2}{c}{validation}
& \multicolumn{2}{c}{test} \\
& Acc. & $F_1$
& Acc. & $F_1$ \\
\midrule
left-to-right &
 28.14&
 17.54&
 28.07&
 17.43 \\
tree &
68.23 &
68.74 &
68.74 &
69.12 \\
tree+budget &
{\bf 82.35}  &
{\bf 82.59}  &
{\bf 82.75}  &
{\bf 82.95}  \\
\bottomrule
\end{tabular}
\end{table}

First, we consider a factor graph with a single non-projective \textsf{TREE}
factor: in this case, {\lpsmap} reduces to a {\smap} baseline.
Motivated by multiple observations that {\smap} and similar latent structure learning
methods tend to learn trivial trees \citep{adina}
we next consider overlaying \textbf{constraints} in the
form of \textsf{BUDGET} factors on top of the \textsf{TREE} factor.
For every possible head $i$, we include a \textsf{BUDGET} factor allowing at
most five of the possible outgoing arcs
$(\mmg_{i \rightarrow 1}, \dots, \mmg_{i \rightarrow L})$ to be selected.

\paragraph{Results.} Figure~\ref{fig:listops} confirms that, unsurprisingly, the
baseline with access to gold dependency structure quickly learns to predict
perfectly, while the simple left-to-right baseline cannot progress.
{\lpsmap} with
\textsf{BUDGET} constraints on the modifiers outperforms {\smap} by over 10
percentage points (Table~\ref{tab:listops}).

\subsection{Natural language inference\\with decomposable structured attention}
We now turn to the task of natural language inference, using \lpsmap to uncover
hidden alignments for structured attention networks.
Natural language inference is a pairwise classification task. Given a
\emph{premise} of length $m$, and a \emph{hypothesis} of length $n$, the pair
must be classified into one of
three possible relationships: entailment, contradiction, or neutrality. We use
the English language SNLI and MultiNLI datasets \citep{snli,multinli}, with the
same preprocessing and splits as \citet{sparsemap}.
\paragraph{Model architecture.}
We use the model of \citet{decomp} with no
intra-attention. The model computes a joint attention score matrix $\bm{S}$ of size
$m \times n$, where $s_{ij}$ depends only on $i$th word in the premise
and the $j$th word in the hypothesis (hence \emph{decomposable}).
For each premise word $i$, we apply \textbf{softmax} over the $i$\textsuperscript{th}
row of $\bm{S}$
to get a weighted average of the hypothesis. Then, similarly, for each
hypothesis word $j$, we apply softmax over the $j$\textsuperscript{th} row
of $\bm{S}$ yielding a representation of the premise. From then on,
each word embedding is combined with
its corresponding weighted context using an affine function, the results are
sum-pooled and passed through an output multi-layer perceptron to make a
classification.
We propose replacing the independent softmax attention with structured,
joint attention, normalizing over both rows and columns \emph{simultaneously} in
several different ways, using {\lpsmap} with scores $\ppr_{ij} = s_{ij}$.
We use frozen GloVe embeddings \citep{glove}, and all our models have 130k parameters (\cf \suppref{experimental}).

\paragraph{Factor graphs.} Assume $m \leq n$.
First,
like \citet{sparsemap}, we consider a
\textbf{matching} factor $f$:%
\begin{equation}
\Mpo_f\!=\!
\Big\{ \mg \in [0,1]^{mn};\!\!
\sum_{j \in [n]} \mmg_{ij} = 1,
\sum_{i \in [m]} \mmg_{ij} \leq 1
\Big\}.
\end{equation}%
\begin{table}[t]
\caption{NLI accuracy scores with structured attention. The {\lpsmap} models
perform competitively.\label{tab:nli}}
\small%
\centering%
\begin{tabular}{l r r r r}
\toprule
& \multicolumn{2}{c}{SNLI} & \multicolumn{2}{c}{MultiNLI} \\
& valid & test & valid & test \\
\midrule
softmax & 84.44 & 84.62 & 70.06 & 69.42 \\
matching & 84.57 & 84.16 & 70.84 & 70.36 \\
LP-matching & {\bf 84.70} & {\bf 85.04} & 70.57 & 70.64 \\
LP-sequential &83.96 & 83.67 & {\bf 71.10}& {\bf 71.17} \\
\bottomrule
\end{tabular}
\end{table}%
When $m=n$, linear maximization on this constraint set corresponds to the linear
assignment problem, solved by
%\citet{km,lapjv}.
the Kuhn-Munkres \citep{km} or Jonker-Volgenant \citep{lapjv}
algorithms, and
the solution is a doubly stochastic matrix.
When $m < n$, the scores can be padded with $-\infty$ to a square
matrix prior to invoking the algorithm. A linear maximization thus takes
$\mathcal{O}(n^3)$, and this instantiation of structured matching attention can
be tackled by {\smap}. Next we consider a relaxed equivalent formulation
which we call \textbf{LP-matching}, as shown in Figure~\ref{fig:match},
with one \textsf{XOR} factor per row and one \textsf{AtMostOne} factor per
column:
\begin{equation}
\begin{aligned}
\mathcal{F}=~&\{\textsf{XOR}(\mmg_{i1}, \dots, \mmg_{in}): i \in [m]\}
\\&\cup \{\textsf{AtMostOne}(\mmg_{1j}, \dots, \mmg_{mj}): j \in [n]\}
\label{eqn:atmostone-factors}
\end{aligned}
\end{equation}
Each subproblem can be solved in $\mathcal{O}(n)$ for a total complexity of
$\mathcal{O}(n^2)$ per iteration (\cf Appendix~\ref{supp:specialized}).
While more iterations may be
necessary to converge, the finer-grained approach might make
faster progress, yielding more useful latent alignments.
Finally, we consider a more expressive joint alignment that encourages
continuity. Inspired by the sequential alignment of \citet{sparsemap},
we propose a bi-directional model called \textbf{LP-sequence}, consisting of
a coarse, linear-chain Markov factor
\citep[with MAP provided by the Viterbi algorithm;][]{Rabiner1989}
parametrized by a single transition score
$\eta_N$ for every pair of alignments $(i, j)-(i
+ 1, j \pm 1)$. By itself, this factor may align multiple premise words to
the same hypothesis word.
We symmetrize it by overlaying
$m$ \textsf{AtMostOne} factors, like in
\eqnref{eqn:atmostone-factors}, ensuring each hypothesis word is aligned
on average to at most one premise word. Effectively, this results in a
sequence tagger constrained
to use each of the $m$ states at most once.
For both {\lpsmap} approaches, we rescale the result by row sums to ensure
feasibility.

\begin{figure}[t]%
\centering\includegraphics[width=.48\textwidth]{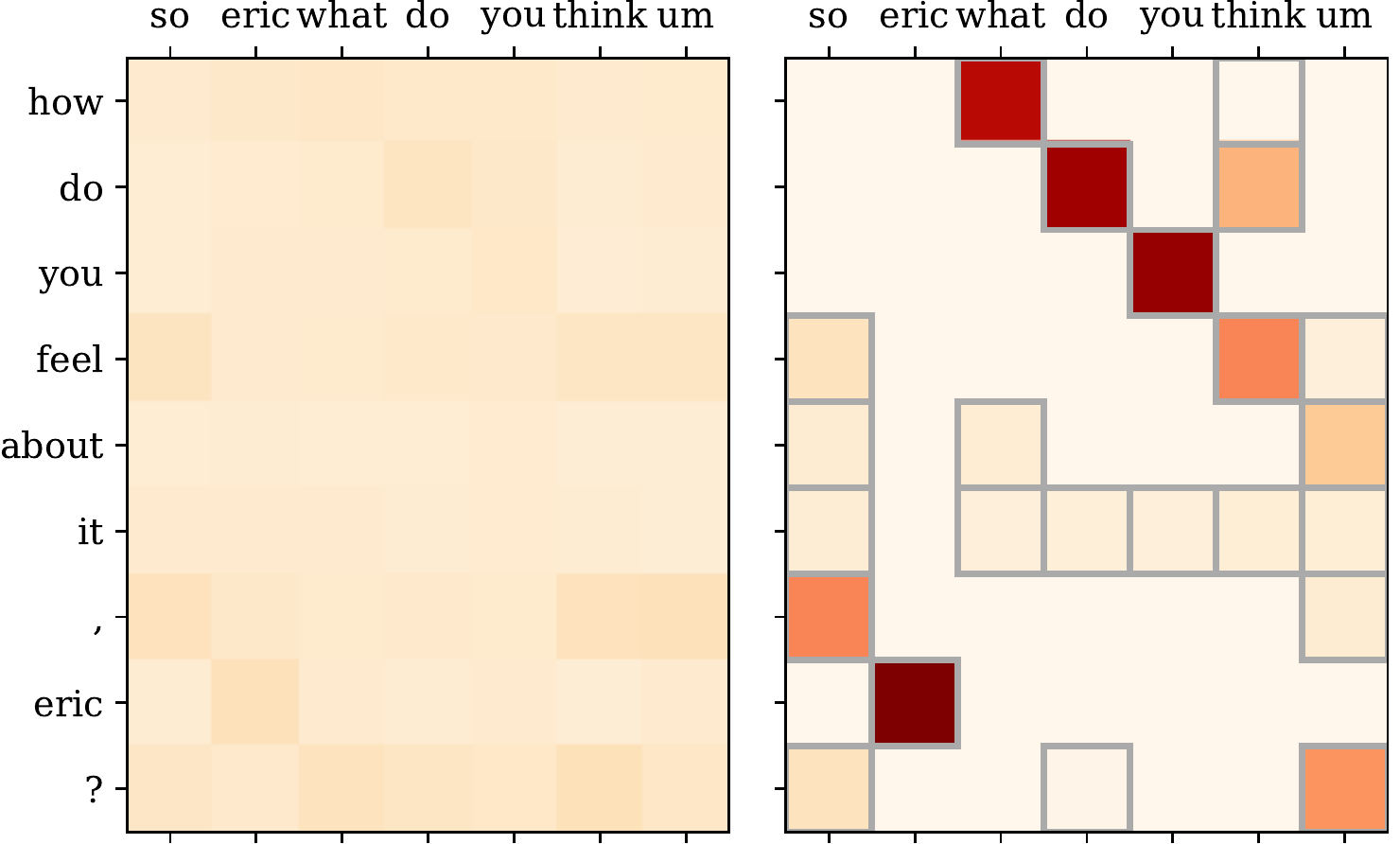}%
\caption{Attention induced using softmax (left) and LP-SparseMAP
sequential (right) on a MultiNLI example.
With this inductive bias, \lpsmap learns
a bi-directional alignment anchoring longer
phrases.\label{fig:multinli}}
\end{figure}

\paragraph{Results.} Table~\ref{tab:nli} reveals that LP-matching
is the best performing mechanism on SNLI, and LP-sequential on MultiNLI.
The $\eta_N$ transition score learned by LP-sequential is 1.6 on SNLI
and 2.5 on MultiNLI, and Figure~\ref{fig:multinli} shows an example of the useful inductive bias it learns.
On both datasets, the relaxed LP-matching outperforms the
coarse matching factor, suggesting that, indeed, equivalent parametrizations of
a model may perform differently when not run until convergence.

\begin{table}[t]
\small%
\centering%
\caption{Multilabel classification test $F_1$ scores.\label{tab:ml}}
\begin{tabular}{l r r}
\toprule
& bibtex & bookmarks \\
\midrule
Unstructured & 42.28 & 35.76 \\
Structured hinge loss & 37.70 & 33.26 \\
{\lpsmap} loss & {\bf 43.43} & {\bf 36.07} \\
\bottomrule
\end{tabular}
\end{table}
\subsection{Multilabel classification}
Finally, to confirm that LP-SparseMAP is also suitable as in the supervised
setting, we evaluate on the task of multilabel classification. Our
factor graph has $k$ binary variables (one for each label), and a pairwise
factor for every label pair:
\begin{equation}
\mathcal{F}=\{ \textsf{PAIR}(\mmg_i, \mmg_j; \ppr_{ij}) : 1 \leq i < j \leq k\}.
\end{equation}
This yields the standard fully-connected pairwise MRF:
\begin{equation}
\DP{\pr}{\mg} = \sum_i \mmg_i \ppr_i + \sum_{i < j} \mmg_{i} \mmg_{j} \ppr_{ij}.
\end{equation}
\paragraph{Neural network parametrization.} We use a 2-layer multi-layer perceptron
to compute the score for each variable. In the structured models, we have an
additional $\hlf k (k - 1)$ parameters for the co-occurrence score of every pair
of classes. We compare an unstructured baseline (using the binary logistic loss
for each label), a structured hinge loss (with LP-MAP inference) and a {\lpsmap} loss model.
We solve LP-MAP using \adq and {\lpsmap} with our proposed algorithm
(\cf Appendix~\ref{supp:experimental}).

\paragraph{Results.} Table~\ref{tab:ml} shows the example $F_1$ score on the test set
for the \emph{bibtex} and \emph{bookmarks} benchmark
datasets \citep{katakis2008multilabel}. The structured hinge loss model
is worse than the unstructured (binary logistic loss) baseline;
the {\lpsmap} loss model
outperforms both. This suggests that the {\lpsmap} loss is
promising for structured output learning.
We note that, in strictly-supervised setting, approaches that
blend inference with learning \citep[\eg,][]{chen2015learning,tang2016bethe}
may be more efficient; however, LP-SparseMAP can work both as a hidden layer and
a loss, with no redundant computation.
\section{Related work}
\paragraph{Differentiable optimization.} The most related research direction
involves bi-level optimization, or \emph{argmin differentiation}
\citep{gould,djolonga};
Typically, such research assumes problems are expressible in a standard form, for instance
using quadratic programs \citep{optnet} or generic disciplined convex programs
\citep[Section 7,][]{amos2019differentiable,amoscvx,diffcone}.
We take inspiration from this line of work by developping \lpsmap as a flexible
domain-specific language for defining latent structure.
The generic approaches are not applicable for the typical optimization problems arising
in structured prediction, because of the intractably large number of constraints
typically necessary, and the difficulty of formulating many problems in standard
forms.
Our method instead assumes interacting through the problem through local oracle
algorithms, exploiting the structure of the factor graph and allowing for more efficient handling of
coarse factors and logic constraints via \emph{nested}
subproblems.

\paragraph{Latent structure models.}
Our motivation and applications are mostly focused on learning with
latent structure. Specifically, we are interested in global optimization
methods, which require marginal inference or similar
relaxations \citep{kim-structuredattn,lapata,caio-iclr,caio-acl,sparsemap},
rather than incremental methods based on policy gradients \citep{rlspinn}.
Promising methods exist for approximate marginal inference in factor graphs with MAP
calls \citep{belanger2013marginal,barrierfw,tang2016bethe},
relying on entropy approximation
penalties.
Such approaches focus on supervised structure prediction, which is not our main
goal; and their backward passes has not been studied to our knowledge.
Importantly, as these penalties are non-quadratic, the active set algorithm does not
apply, falling back to the more general variants of Frank-Wolfe.
The active set algorithm is a key ingredient of our work, as it exhibits fast finite convergence,
finds sparse solutions and -- crucially -- provides precomputation of the matrix inverse
required in the backward pass \citep{sparsemap}.
In contrast, the quadratic penalty \citep{smooth_and_strong,sparsemap} is more
amenable to optimization, as well as bringing other sparsity benefits.
The projection step of \citet{spigot} can be cast as a \smap problem,
thus our algorithm can be used to also extend their method to arbitrary factor graphs.
For pairwise MRFs (a class of factor graphs), differentiating belief propagation, either through
unrolling or perturbation-based approximation, has been studied \citep{ves,domke}.
Our approach instead computes \emph{implicit} gradients, which is more
efficient, thanks to quantities precomputed in the forward pass, and in some
circumstances has been shown to work better \citep{metalearning}.
Finally, MRF-based approaches have not been explored in the presence of
logic constraints or coarse factors, while our formulation is
built from the beginning with such use cases in mind.

\section{Conclusions}
We introduced \lpsmap, an extension of \smap to sparse
differentiable optimization in any factor graph, enabling neural hidden
layers with arbitrarily complex structure, specified using a familiar
domain-specific language. We have shown \lpsmap to outperform \smap for latent
structure learning, %and its corresponding loss function to outperform the
and outperform the structured hinge for structured output learning.
We hope that our toolkit empowers future research on latent structure,
leading to powerful models based on domain knowledge.
In future work, we shall investigate further applications
where expertise about the domain structure, together with minimal
self-supervision deployed via the \lpsmap loss, may lead to
data-efficient learning, even for more expressive models without artificial bottlenecks.

\section*{Acknowledgements}

We are grateful to
Brandon Amos,
Mathieu Blondel,
Gon\c{c}alo Correia,
Caio Corro,
Erick Fonseca,
Pedro Martins,
Tsvetomila Mihaylova,
Nikita Nangia,
Fabian Pedregosa,
Marcos Treviso,
and the reviewers, for their valuable feedback and discussions.
This work is built on open-source software; we acknowledge the scientific Python
stack \citep{python,numpy,nparray,scipy,cython} and the developers of Eigen \citep{eigen}.
This work was supported by the European Research Council (ERC StG DeepSPIN
758969), by the Funda\c{c}\~ao para a Ci\^encia e Tecnologia
through contracts UID/EEA/50008/2019 and CMUPERI/TIC/0046/2014 (GoLocal),
and by the MAIA project, funded by the P2020 program under contract number 045909.

\bibliographystyle{apalike}

\input{paper.bbl}
\clearpage
\appendix
\onecolumn
\begin{center}
{\huge \textbf{Supplementary Material}}
\end{center}

\section{Separable reformulation of LP-SparseMAP}\label{supp:lpsmap}

\begin{lemma}\label{lemma:ctilde}
Let $\bm{\delta}$, $\bm{D}$, $\CCs$, $\MMs$ defined as in Proposition~\ref{prop:separable}.
Let $\bm{S} = \diag(\bm{\delta})$.
Then,
\begin{enumerate}[label=(\roman*)]
    \item $\CC^\top \CC = \bm{S}^2$
    \item $\CCs = \CC \bm{S}^{-1}$;
    \item $\CCs^\top\CCs = \bm{I}$;
    \item For any feasible pair $(\mg, \p)$,
        $\mg = \CCs^\top \MMs \p$, and
        $\| \mg \| = \| \MMs\p \|.$
\end{enumerate}
\end{lemma}
\begin{proof}
(i)
The matrix $\CC$, which expresses the agreement constraint
$\CC \mg = \MM \p$, is a stack of selector matrices, in other words, its
sub-blocks are either the identity $\bm{I}$ or the zero matrix $\bm{0}$.
We index its rows by
pairs $(f, k): f \in \FF, k \in [d_f]$,
and its columns by $j \in [d]$.
Denote by $\agr{f}{k}{j}$ the fact that the $k$\textsuperscript{th}
variable under factor $f$ is $\mg_j$.  Then,
$(\CC)_{(f, k), j} = \iv{\agr{f}{k}{j}}$. We can then explicitly compute
\[(\CC^\top \CC)_{ij} = \sum_{f \in \FF} \sum_{k \in [d_f]}
\iv{\agr{f}{k}{i}}\iv{\agr{f}{k}{j}}.\] If $i \neq j$,
$\iv{\agr{f}{k}{i}}\iv{\agr{f}{k}{j}}=0$, so
$(\CC^\top\CC)_{ij} = \begin{cases}\deg(j)&i=j,\\0,&\text{o.w.}\end{cases}
=\bm{S}^2$.

(ii) By construction,
$\bm{D}_{(f,k),(f,k)} = (\CC \bm{\delta})_{(f,k)} =
\sum_{i \in [d]}
\iv{\agr{f}{k}{i}}
\sqrt{\deg(i)}
= \sqrt{\deg(j)},$
for the unique variable $j$ with $\agr{f}{k}{j}$.
Thus,
\[
(\bm{D}^{-1}\CC)_{(f,k),j}=
\iv{\agr{f}{k}{j}}
\sqrt{\deg(j)}
=
(\CC\bm{S}^{-1})_{(f,k),j}.
\]

(iii) It follows from (i) and (ii) that $\CCs^\top \CCs = \bm{S}^{-1}\CC^\top\CC \bm{S}^{-1}=
\bm{S}^{-1}\bm{S}^2\bm{S}^{-1}=\bm{I}$.

(iv) Since $\bm{D}$ is full-rank, the feasibility condition is equivalent to
$\CCs \mg = \MMs \p$. Left-multiplying by $\CCs^\top$ yields $\mg = \CCs^\top
\MMs \p$. Moreover, $\|\MMs \p\|^2 = \| \CCs \mg \|^2 = \mg^\top \CCs^\top
\CCs \mg = \|\mg\|^2.$
\end{proof}

\section{Derivation of updates and comparison to LP-MAP}\label{supp:ad3qp}

Recall the problem we are trying to minimize, from \eqnref{eqn:lpsmapsep}:

\begin{equation}
\underset{\mg,~\p}{\text{maximize}}
\sum_{f\in\FF} \DP{\pr_f}{\AA_f \p_f} - .5 \| \MMs_f \p_f \|^2
~~~\text{subject to}~~~
\p \in \simplex_{f_1} \times \simplex_{f_2} \times \dots \times \simplex_{f_n},
~ \CCs \mg = \MMs \p.
\end{equation}

Since the simplex constraints are separable, we may move them to the objective,
yielding
\begin{equation}
\label{eqn:lpsmapsep-admm}
\underset{\mg,~\p}{\text{maximize}}~~
\sum_{f\in\FF} \DP{\pr_f}{\AA_f \p_f} - .5 \| \MMs_f \p_f \|^2
- \iota_{\simplex_f}(\p_f)
\quad\text{subject to}\quad
~ \CCs \mg = \MMs \p.
\end{equation}

The $\gamma$-augmented Lagrangian of problem \ref{eqn:lpsmapsep-admm} is
\begin{equation}
\begin{aligned}
    \mathcal{L}_\gamma(\mg, \p. \lbd) &=
\sum_{f\in \FF} \Big(
    \DP{\pr_f}{\AA_f \p_f} - .5 \| \MMs_f \p_f \|^2 - \iota_{\simplex_f}(\p_f)
\Big)
    - \DP{\lbd}{\CCs \mg - \MMs \p}
    - \frac{\gamma}{2} \| \CCs\mg - \MMs \p \|^2. \\
\end{aligned}
\end{equation}

The solution $\mg^\star, \p^\star, \lbd^\star$ is a saddle point of the
Lagrangian, \ie, a solution of
\begin{equation}\label{eqn:saddle}
\min_{\lbd} \max_{\p, \mg} \mathcal{L}_\gamma(\mg,\p,\lbd)
\end{equation}
ADMM optimizes \eqnref{eqn:saddle} in a block-coordinate fashion;
we next derive each block update.

\subsection{Updating \texorpdfstring{{\boldmath $p$}}{p}}
We update $\p_f$ for each $f \in \FF$ \textbf{independently} by solving:
\begin{equation}
\begin{aligned}
\p_f^{(t)} \leftarrow \argmax_{\p_f}~&
\mathcal{L}_\gamma(\mg^{(t-1)},\p,\lbd^{(t-1)})\\
\end{aligned}
\end{equation}

Denoting $\pr_f = [\pr_{f,M}, \pr_{f,N}]$, we have that
\[
\DP{\pr_f}{\AA_f \p_f} =
\DP{\pr_{f,M}}{\MM_f \p_f} +
\DP{\pr_{f,N}}{\NN_f \p_f} =
\DP{\bm{D}_f \pr_{f,M}}{\MMs_f \p_f} +
\DP{\pr_{f,N}}{\NN_f \p_f}
\]
The $\gamma$-augmented term regularizing the subproblems toward the
current estimate of the global solution $\mg^{(t-1)}$ is
\[
\frac{\gamma}{2} \| \CCs_f \mg^{(t-1)} - \MMs_f \p_f \|^2
= \frac{\gamma}{2} \| \MMs_f \p_f \| - \gamma \DP{\CCs_f \mg^{(t-1)}}{\MMs_f \p_f} +
\text{const}
\]
For each factor, the subproblem objective is therefore:
\begin{equation}\label{eqn:subprob}
\begin{aligned}
f(\p_f) &= \DP{\pr_f}{\AA_f \p_f} - \DP{\lbd^{(t)}_f}{\MMs_f \p_f}
- \frac{\gamma}{2} \| \CCs_f \mg^{(t-1)} - \MMs_f \p \|^2
- \frac{1}{2} \|\MMs_f \p\|^2 \\
&=
\DP{\bm{D}_f \pr_{f,M} - \lbd^{(t-1)}_f + \gamma \CCs_f \mg^{(t-1)}}{\MMs_f \p_f}  +
\DP{\pr_{f,N}}{\NN_f\p_f} -
\frac{1+\gamma}{2} \| \MMs_f \p_f\|^2 + \text{const} \\
&\propto
\DP{\widetilde{\pr}_{f,M}}{\MMs_f \p_f} +
\DP{\widetilde{\pr}_{f,N}}{\NN_f \p_f} -
\frac{1}{2} \| \MMs_f \p_f\|^2 + \text{const}.
\end{aligned}
\end{equation}
This is exactly a SparseMAP instance with
$\widetilde{\pr}_{f,M} = \frac{1}{1+\gamma}\big(\bm{D}_f \pr_{f,M} - \lbd^{(t-1)}_f + \gamma \CCs_f
\mg^{(t-1)}\big)$ and
$\widetilde{\pr}_{f,N} = \frac{1}{1+\gamma}\pr_{f,N}$.

\textbf{Observation.}
For comparison, when solving LP-MAP with \adq, the subproblems minimize the objective
\begin{equation}
\begin{aligned}
f(\p_f) &= \DP{\pr_f}{\AA_f \p_f} - \DP{\lbd^{(t)}_f}{\MM_f \p_f}
- \frac{\gamma}{2} \| \CC_f \mg^{(t)} - \MM_f \p_f \|^2 \\
&= \DP{\pr_{f,M} - \lbd^{(t)}_f + \gamma \CC_f \mg^{(t)}}{\MM_f \p_f}
+ \DP{\pr_{f,N}}{\NN_f \p_f} - \frac{\gamma}{2} \|\MM_f\p_f \|^2,
\end{aligned}
\end{equation}
so the $\p$-update is a SparseMAP instance with
$\widetilde{\pr}_{f,M} = \frac{1}{\gamma} \big(\pr_{f,M} - \lbd^{(t)}_f + \gamma \CC_f \mg^{(t)}\big)$
and
$\widetilde{\pr}_{f,N} = \frac{1}{\gamma} \pr_{f,N}$.
Notable differences is the scaling by $1+\gamma$ instead of $\gamma$
(corresponding to the added regularization), and the diagonal degree
reweighting.

\subsection{Updating \texorpdfstring{\boldmath $\mu$}{the marginals}}
We must solve
\begin{equation}
\begin{aligned}
    \mg^{(t)} &\leftarrow \argmax_{\mg}~
\mathcal{L}_\gamma(\mg,\p^{(t)},\lbd^{(t-1)})\\ %
&= \argmin_{\mg}~ \frac{\gamma}{2} \| \CCs \mg - \MMs \p^{(t)} \|^2
+ \DP{\CCs^\top\lbd^{(t-1)}}{\mg}.
\end{aligned}
\end{equation}
This is an unconstrained problem. Setting the gradient of the
objective to $\bm{0}$, we get
\begin{equation}
\begin{aligned}
    \bm{0} &\setto \gamma \CCs^\top (\CCs \mg - \MMs \p^{(t)}) +
\CCs^\top\lbd^{(t-1)} \\
&= \gamma(\mg - \CCs^\top\MMs\p^{(t)}) + \CCs^\top\lbd^{(t-1)}
\end{aligned}
\end{equation}
with the unique solution
\begin{align}
    \mu^{(t)} &\leftarrow \CCs\MMs\p^{(t)} - \frac{1}{\gamma} \CCs^\top
    \lbd^{(t-1)}\label{eqn:mu_update_orig} \\
    &= \CCs\MMs\p^{(t)},
\end{align}
where the last step follows from the fact that
our resulting algorithm
maintains the invariant $\CCs^\top\lbd^{(\cdot)}=0$,
as we show in the next section.

\subsection{Updating the Lagrange multipliers}
Since $\mathcal{L}_\gamma$ is linear in $\lbd$, $\min_{\lbd}
\mathcal{L}_\gamma(\lbd) = -\infty$, therefore we may not globally minimize \wrt
$\lbd$. Instead, we make only a small gradient step:
\begin{equation}
\lbd^{(t)} \leftarrow \lbd^{(t-1)} + \gamma \big(\CCs \mg^{(t)} - \MMs \p^{(t)}
\big).
\end{equation}
As promised, we inspect below the value of $\CCs^\top\lbd$ under this update
rule.
\begin{equation}
\begin{aligned}
    \CCs^\top\lbd^{(t)} &=
    \CCs^\top\lbd^{(t-1)} + \gamma(\cancel{\CCs}^\top\cancel{\CCs}\mg^{(t)} -
    \CCs^\top\MMs\p^{t}) \\
    &= \CCs^\top\lbd^{(t-1)} + \gamma\big(\mg^{(t)} - (\mg^{(t)} +
\frac{1}{\gamma}\CCs^\top\lbd^{(t-1)})\big)
\quad\text{(from Eq.~\ref{eqn:mu_update_orig})}\\
&= \CCs^\top\lbd^{(t-1)} - \frac{\gamma}{\gamma} \CCs^\top\lbd^{(t-1)} =0.
\end{aligned}
\end{equation}

\section{Backward pass}
\subsection{SparseMAP}\label{supp:smap_grad}

As a reminder, we repeat here the form of the \smap Jacobian \citep{sparsemap},
along with a brief derivation. This result plays an important role in \lpsmap
backward pass.

\begin{proposition}
Given a structured problem with $\AA= [\MM, \NN]$,
denote the \smap solution for input scores $\pr = [\pr_M, \pr_N]$ as
$\mg$ where
\begin{equation}\label{eq:single_smap}
(\mg, \p) =
\argmax_{\substack{\mg = \MM \p\\ \p \in \simplex}}
\DP{\pr}{\AA\p}
- \frac{1}{2} \| \mg \|^2.
\end{equation}

Let $\mathcal{S} = \{y_1, \dots, y_k\} \subset \YY$ denote the support set of
selected structures, and denote
$\Mb \coloneqq \MM_\mathcal{S} \in \reals^{d_M \times |\mathcal{S}|}$,
$\bar{\NN} \coloneqq \NN_\mathcal{S} \in \reals^{d_N \times |\mathcal{S}|}$, and
\begin{equation}
\ZZ = (\Mb^\top \Mb)^{-1},\quad
\bm{z} = \ZZ \bm{1},\quad
\QQ = \ZZ - \frac{\bm{z}\bm{z}^\top}{\bm{1}^\top\bm{z}}.
\end{equation}

Then, we have
\begin{equation}\label{eq:smap_grad}
\pfrac{\mg}{\pr_M}(\pr_M, \pr_N) = \MQM,
\qquad
\pfrac{\mg}{\pr_N}(\pr_M, \pr_N) = \Mb \QQ \bar{\NN}.
\end{equation}
\end{proposition}

\begin{proof}
Rewrite the optimization problem in Eq.~\ref{eq:single_smap} in terms of
a convex combination of structures:
\begin{equation}\label{eq:simplex_smap}
\text{minimize}~\DP{\pt}{\p} - \frac{1}{2} \| \MM\p \|^2
\quad\text{subject to}~
\p \in \simplex.
\end{equation}
The Lagrangian is given by
\begin{equation}
\mathcal{L}(\p, \bm{\nu}, \tau) =
\frac{1}{2} \| \MM\p \|^2
-\DP{\pt - \tau\bm{1} - \bm{\nu}}{\p}.
\end{equation}

The solution $\p$ is sparse with nonzero coordinates $\mathcal{S}$. Small
changes to $\pt$ only lead to changes in $\mathcal{S}$ on a
measure-zero set of critical tie-breaking points, and there is always a
direction of change that leaves $\mathcal{S}$ unchanged.
We may thus assume that $\mathcal{S}$ does not change with small changes to
$\pt$, yielding the Jacobian at most points, and a generalized Jacobian
otherwise \citep{clarke}.

From complementary slackness, $\bar{\bm{\nu}} = \bm{0}$,
so the conditions $\nabla_{\bar{\p}} \mathcal{L} \setto \bm{0}$ and $\bm{1}^\top \bar{\p}
\setto 1$ can be written as
\begin{equation}
\begin{bmatrix}
\Mb^\top \Mb & \bm{1} \\
\bm{1}^\top  & 0 \\
\end{bmatrix}
~
\begin{bmatrix}
\bar{\p} \\
\tau \\
\end{bmatrix}
~=~
\begin{bmatrix}
\bar{\pt} \\
1 \\
\end{bmatrix}.
\end{equation}
Therefore, differentiating w.r.t.\ $\bar{\pt}$, the Jacobians
$\pfrac{\bar{\p}}{\bar{\pt}}$ and
$\pfrac{\tau}{\bar{\pt}}$ must satisfy
\begin{equation}
\begin{bmatrix}
\Mb^\top \Mb & \bm{1} \\
\bm{1}^\top  & 0 \\
\end{bmatrix}
~
\begin{bmatrix}
\pfrac{\bar{\p}}{\bar{\pt}}\\
\pfrac{\tau}{\bar{\pt}}\\
\end{bmatrix}
~=~
\begin{bmatrix}
\bm{I} \\
0 \\
\end{bmatrix}.
\end{equation}
Denote by $\ZZ \coloneqq (\Mb^\top \Mb)^{-1}, \bm{z} = \ZZ\bm{1}, t \coloneqq
\bm{1}^\top\bm{z},
\QQ = \ZZ - \frac{\bm{zz}^\top}{t}.$
Using block-matrix inversion,
\begin{equation}\label{eq:block_inverse}
\begin{bmatrix}
\Mb^\top \Mb & \bm{1} \\
\bm{1}^\top  & 0 \\
\end{bmatrix}^{-1} =
\begin{bmatrix}
\QQ & \nicefrac{\bm{z}}{t} \\
\nicefrac{\bm{z}^\top}{t} & -\nicefrac{1}{t} \\
\end{bmatrix}.
\end{equation}
Therefore, $\pfrac{\bar{\p}}{\bar{\pt}} = \QQ$.
Since $\mg_M = \Mb{\bar{\p}}$ and $\bar{\pt} = \Mb^\top \pr_M + \bar{\NN}^\top
\pr_N$, the chain rule gives Eq.~\ref{eq:smap_grad}.
Importantly, when using the active set method for computing the \smap solution
\citep{sparsemap}, the inverse in Eq.~\ref{eq:block_inverse}, and thus $\QQ$, is
precomputed incrementally during the forward pass, and thus readily available
for no extra cost..

\end{proof}

\subsection{\lpsmap}\label{supp:lpsmap_grad}

\begin{proof}
Given variable scores $\pr_M$ and factor scores $\pr_{f,N}$, we construct a
vector $\pt = \Mb^\top \CCs \pr_M + \NN \pr_{f,N}$.
To derive the backward pass, we start from the Lagrangian with simplex constraints:
\begin{equation}
    \mathcal{L}(\mg, \p. \lbd. \bm{\tau}, \bm{\nu}) =
    \DP{\pt}{\p}
- \frac{1}{2} \| \MMs \p \|^2
- \DP{\lbd}{\CCs \mg - \MMs \p}
- \DP{\bm{\tau}}{\bm{B}\p - \bm{1}}
- \DP{\bm{\nu}}{\p}.
\end{equation}
where $\bm{B}$ is a matrix with row-vectors $\bm{1}$ along the diagonal (so that
 $\bm{B}\p = [\cdots, \bm{1}\p_f, \cdots]$).
For any feasible $(\p, \mg)$ we have that  $\| \MMs \p \|^2 = \| \mg \|^2$,
so we may rewrite the Lagrangian as:
\begin{equation}
    \mathcal{L}(\mg, \p. \lbd. \bm{\tau}, \bm{\nu}) =
    \DP{\pt}{\p}
- \frac{1}{4} \| \MMs \p \|^2
- \frac{1}{4} \| \mg \|^2
- \DP{\lbd}{\CCs \mg - \MMs \p}
- \DP{\bm{\tau}}{\bm{B}\p - \bm{1}}
- \DP{\bm{\nu}}{\p}.
\end{equation}
The corresponding optimality conditions are
\begin{align}
    \bm{0} \setto \nabla_{\p_f} \mathcal{L} &=
    \pt
    - .5 \MMs_f^\top \MMs_f \p_f
    + \MMs_f^\top \lbd_f
    - \tau_f\bm{1} - \bm{\nu}_f \quad\text{for all}~f \in \FF,\\
    \bm{0} \setto \nabla_{\mg} \mathcal{L} &=
    -.5 \mg - \CCs^\top \lbd \\
    \bm{0} \setto \nabla_{\lbd} \mathcal{L} &= \CCs\mg - \MMs\p\\
    \bm{0} \setto \nabla_{\bm{\tau}} \mathcal{L} &= \bm{B}\p - \bm{1}
\end{align}
along with $\bm{\nu} \geq 0, \p \geq 0,$ and the complementarity slackness
conditions $\DP{\bm{\nu}}{\p} = \bm{0}$.
As in \suppref{smap_grad}, we observe that the support $\mathcal{S}_f$
of each factor $f$ does not change with small changes to $\bm{\eta}$.
Once again, we use the overbar $\bar\cdot$ to denote the
restriction of a vector or matrix to the (block-wise) support $\mathcal{S}_f$,
resulting in, for instance,
\[
\bar\p > 0 \in \RR^{\sum_f |\mathcal{S}_f|},
\quad
\Mb \in \RR^{\left(\sum_f d_f\right)~\times~\left(\sum_f |\mathcal{S}_f|\right)},
\quad\text{etc.}
\]
On the support,  $\bar{\bm{\nu}}_f$ vanishes, so we rewrite the conditions in terms of $\bar{\p}$.
In matrix form,
\begin{equation}
\begin{bmatrix}
.5 \Mb^\top\Mb & \BB^\top& \bm{0}   & -\Mb^\top  \\
\BB            & \bm{0}  & \bm{0}   & \bm{0}     \\
\bm{0}         & \bm{0}  & .5\bm{I} & \CCs^\top  \\
-\Mb           & \bm{0}  & \CCs     & \bm{0}     \\
\end{bmatrix}
\left[\begin{array}{c}\bar\p\\\bm{\tau}\\\mg\\\lbd\end{array} \right]
=\left[ \begin{array}{c} \bar{\pt} \\\bm{1}\\\bm{0}\\\bm{0}\end{array} \right]
\end{equation}
Differentiating w.r.t.\ $\bar{\pt}$ yields
\begin{equation}
\begin{bmatrix}
.5 \Mb^\top\Mb & \BB^\top& \bm{0}   & -\Mb^\top  \\
\BB            & \bm{0}  & \bm{0}   & \bm{0}     \\
\bm{0}         & \bm{0}  & .5\bm{I} & \CCs^\top  \\
-\Mb           & \bm{0}  & \CCs     & \bm{0}     \\
\end{bmatrix}
\left[\begin{array}{c}\JJ_{\bar\p}\\\JJ_{\bm{\tau}}\\\JJ_{\mg}\\\JJ_{\lbd}\end{array} \right]
=\left[ \begin{array}{c} \bm{I} \\\bm{0}\\\bm{0}\\\bm{0}\end{array} \right]
\end{equation}
Observe that the top-left block can be re-organized into a block-diagonal
matrix with blocks with known inverses (similar to Eq.~\ref{eq:block_inverse})
\begin{equation}
\begin{bmatrix}
.5 \Mb_f^\top \Mb_f & \bm{1} \\
\bm{1}^\top  & 0 \\
\end{bmatrix}^{-1} =
\begin{bmatrix}
2\QQ_f & \cdot \\
\cdot & \cdot \\
\end{bmatrix}
\end{equation}
where the values except for the top-left block can be easily obtained in terms
of the blocks of Eq.~\ref{eq:block_inverse}, but this is not necessary,
since all others rows and columns corresponding to $\bm{\tau}$ are zero.

We multiply the top half of the system by this inverse and eliminate
$\bm{\tau}$, leaving
\begin{equation}
\begin{bmatrix}
\bm{I}         & \bm{0}   & -2\QQ\Mb^\top  \\
\bm{0}         & .5\bm{I} & \CCs^\top  \\
-\Mb           & \CCs     & \bm{0}     \\
\end{bmatrix}
\left[\begin{array}{c}\JJ_{\bar\p}\\\JJ_{\mg}\\\JJ_{\lbd}\end{array} \right]
=\left[ \begin{array}{c} 2\QQ \\\bm{0}\\\bm{0}\end{array} \right].
\end{equation}
Multiplying the first row of blocks by $\Mb$, the second by $-2\CC$, gives
\begin{equation}
\begin{bmatrix}
\Mb            & \bm{0}   & -2\Mb\QQ\Mb^\top  \\
\bm{0}         & -\CCs & -2 \CCs \CCs^\top  \\
-\Mb           & \CCs     & \bm{0}     \\
\end{bmatrix}
\left[\begin{array}{c}\JJ_{\bar\p}\\\JJ_{\mg}\\\JJ_{\lbd}\end{array} \right]
=\left[ \begin{array}{c} 2 \Mb \QQ \\\bm{0}\\\bm{0}\end{array} \right].
\end{equation}
Finally, we may add up all rows to reach the expression
\[
\JJ_{\lbd} = -\left(\MQM + \CCs\CCs^\top\right)^{+} \Mb\QQ.
\]
and, since $\JJ_{\mg} = -2\CCs^\top \JJ_{\lbd}$, then
\[
\JJ_{\mg} = 2\CCs^\top\left(\MQM + \CCs\CCs^\top\right)^{+} \Mb\QQ.
\]
The Jacobians we have been solving for so far are w.r.t.\ $\pr$. We first apply the
chain rule to get the Jacobian w.r.t.\ $\pt_M$, giving
\begin{equation}
\begin{aligned}
\pfrac{\mg}{\pr_M} &= \JJ_{\mg} \Mb^\top\CCs \\
&= 2\CCs^\top\left(\MQM + \CCs\CCs^\top\right)^{+} \MQM \CCs \\
&= 2\CCs^\top\left(\JJ_M + \CCs\CCs^\top\right)^{+} \JJ_M \CCs, \\
\end{aligned}
\end{equation}
where $\JJ_M$ is the block-wise Jacobian of each SparseMAP subproblem.

Now, observe that $\CCs \CCs^\top$ and $\JJ_M$ are orthogonal projection
matrices: the former because $\CCs$ is orthogonal, the latter because
$\QQ \Mb^\top\Mb \QQ = \QQ$, since for each block
\begin{equation}
\begin{aligned}
\QQ_f \Mb_f^\top\Mb_f \QQ_f
&=
\left(\ZZ_f - \frac{\bm{z}_f \bm{z}_f^\top}{t_f}\right)
\Mb_f^\top\Mb_f
\left(\ZZ_f - \frac{\bm{z}_f \bm{z}_f^\top}{t_f}\right) \\
&=
\left(\ZZ_f - \frac{\bm{z}_f \bm{z}_f^\top}{t_f}\right)
\left(\bm{I} - \frac{\bm{1} \bm{z}_f^\top}{t_f}\right) \\
&=
\ZZ_f - \frac{\bm{z}_f \bm{z}_f^\top}{t_f}
-\ZZ_f \frac{\bm{1} \bm{z}_f^\top}{t_f}
+\frac{\bm{z}_f \bm{z}_f^\top}{t_f} \frac{\bm{1} \bm{z}_f^\top}{t_f} \\
&= \ZZ_f - \frac{\bm{z}_f \bm{z}_f^\top}{t_f}
-\frac{\bm{z}_f \bm{z}_f^\top}{t_f}
+\frac{t_f \bm{z}_f \bm{z}_f^\top}{t_f^2} \\&= \QQ_f.\\
\end{aligned}
\end{equation}

Orthogonal projection matrices are projection operators onto affine subspaces.
We next invoke a result about the projection onto an \emph{intersection} of
affine subspaces:
\begin{lemma} \citep{projintersect}
Let $\mathcal{A},\mathcal{B}$ denote the affine spaces such that
$\proj_\mathcal{A}(\bm{x}) = \bm{P}_\mathcal{A}\bm{x}$ and
$\proj_\mathcal{B}(\bm{x}) = \bm{P}_\mathcal{B}\bm{x}$.
Then, the projection onto their intersection has the following expressions:
\begin{align}
\proj_{\mathcal{A}\cap\mathcal{B}} &=
\lim_{n \rightarrow \infty}
\bm{P}_\mathcal{B} (\bm{P}_\mathcal{A} \bm{P}_\mathcal{B})^n,
\label{eq:intersect_power}\\
&= 2 \bm{P}_\mathcal{B} (\bm{P}_\mathcal{A} +  \bm{P}_\mathcal{B})^{+} \bm{P}_\mathcal{A}
\label{eq:intersect_pinv}
\end{align}
\end{lemma}

Using this lemma, we may apply
Eq.~\ref{eq:intersect_pinv},
to rewrite the Jacobian as
\begin{equation}
\begin{aligned}
\pfrac{\mg}{\pr_M}
&= 2\CCs^\top\left(\JJ_M + \CCs\CCs^\top\right)^{+} \JJ_M \CCs \\
&= \CCs^\top \Big( 2 \CCs \CCs^\top\left(\JJ_M + \CCs\CCs^\top\right)^{+} \JJ_M \Big) \CCs \\
&= \CCs^\top \bm{P}_{\mathcal{A} \cap \mathcal{B}}~\CCs.
\end{aligned}
\end{equation}
where $\bm{P}_\mathcal{A} = \JJ_M$ and $\bm{P}_\mathcal{B}=\CCs\CCs^\top$.
Then, using the power iteration expression
(Eq.~\ref{eq:intersect_power}),
\begin{equation}
\begin{aligned}
\pfrac{\mg}{\pr_M} &= \lim_{n\rightarrow\infty} \CCs^\top \Big(\CCs\CCs^\top (\JJ_M \CCs\CCs^\top)^n\Big)\CCs \\
&= \lim_{n \rightarrow \infty}
\underbrace{\CCs^\top \CCs}_{\bm{I}}\CCs^\top  (\JJ_M \CCs\CCs^\top)^{n-1} \JJ_M
\CCs\underbrace{\CCs^\top\CCs}_{\bm{I}}\\
&= \lim_{n \rightarrow \infty}
(\CCs^\top \JJ_M\CCs)^n \\
\end{aligned}
\end{equation}
Multiplying both sides by $\CCs^\top \JJ_M \CCs$ leaves the r.h.s.\ unchanged, so
\begin{equation}\label{eq:fixed_point_eig_supp}
\CCs^\top \JJ_M \CCs
\pfrac{\mg}{\pr_M} =
\pfrac{\mg}{\pr_M}.
\end{equation}
Finally, we compute the gradient w.r.t.\ $\pr_N$.
Thus we have

\begin{equation}
\begin{aligned}
\pfrac{\mg}{\pr_N} &= \JJ_{\mg} \Mb^\top\CCs \\
&= 2\CCs^\top\left(\JJ_M + \CCs\CCs^\top\right)^{+} \Mb \QQ \bar{\NN}. \\
&= 2\CCs^\top\left(\JJ_M + \CCs\CCs^\top\right)^{+} \Mb \overbrace{\QQ \Mb^\top
\Mb \QQ}^{\QQ} \bar{\NN}. \\
&= \CCs^\top \bm{P}_{\mathcal{A}\cap\mathcal{B}} \Mb \QQ \bar{\NN}. \\
&= \underbrace{\CCs^\top \bm{P}_{\mathcal{A}\cap\mathcal{B}} \CCs}_{\pfrac{\mg}{\pr_M}}\CCs^\top
\underbrace{\Mb \QQ \bar{\NN}}_{\JJ_N},
\end{aligned}
\end{equation}
\end{proof}

If the actual Jacobians are desired, observe that
Eq.~\ref{eq:fixed_point_eig_supp}
says that the columns of $\pfrac{\mg}{\pr_M}$ are eigenvectors of $\CCs^\top
\JJ_M \CCs$ corresponding to eigenvalue 1. We know that the spectrum commutes,
so the spectrum of $\CCs^\top \JJ_M \CCs$ is equal to that of
$\JJ_u \CCs\CCs^\top$, which is a product of two orthogonal projections, thus its eigenvalues
are between $0$ and $1$ \citep{spectrum1,spectrum2}. (This also shows why power
iteration in Eq.~\ref{eq:intersect_power} converges, since all eigenvalues strictly less than $1$ shrink to
$0$.) We may use Arnoldi iteration to obtain the largest eigenvectors of
$\CCs^\top\JJ_M\CCs$.

\section{Specialized algorithms for common factors} \label{supp:specialized}
Like in AD$^3$, any local quadratic subproblem can be solved via the active set
method provided a local
linear oracle (MAP). However, for some special factors, we can derive more
efficient direct algorithms. Many such factors involve logical operations
and constraints which are essential building blocks for expressive inference
problems. We extend the derivations for logic and pairwise factors of
\adq \citep{ad3}, nontrivially, in two ways: first, to accommodate the
\textbf{degree reweighting} needed for \lpsmap{}, and second, to derive \textbf{efficient
expressions for the local backward passes}. Indeed, a useful check is that
our expressions in the case of $\delta_j = 1$ for all $j$ (\ie, when the factor
is alone in the graph) correspond exactly to the non-reweighted QP solutions
derived by \citet{ad3}.

Consider a constraint factor $f$ over $d$ boolean variables. In this case there
are no additional variables, so that the subproblem on line
\ref{line:ad3qp-smap} of Algorithm~\ref{alg:ad3qp} becomes simply:
\begin{equation}\label{eqn:subprob_nov}
\begin{aligned}
\mathrm{minimize}\quad& \hlf \| \widetilde{\pr_f} - \MMs_f \p_f \|^2_2\\
\mathrm{subject\,to}\quad& \p_f \in \simplex_f.
\end{aligned}
\end{equation}
Since it enforces constraints over boolean variables, the allowable set
of assignments (\ie, columns of $\MM_f$) is a subset of $\{0, 1\}^d$.
Therefore, for any $\p_f \in \simplex_f$, we have $\MM_f \p_f \in [0,
1]^d$ as a convex combination of zero-one vectors. Recalling that $\MMs_f
= \bm{D}^{-1}_f \MM_f$ with $\bm{D}_f = \diag(\bm{\delta}_f)$, with
$(\bm{\delta}_f)_i =\sqrt{\deg(i)}$,
we introduce the variable $\mg_f = \MMs_f \p_f.$ We have that $\bm{D}_f \mg_f =
\MM_f \p_f \in [0, 1]^d$.
Since we are focusing on a single factor, we will next drop the subscript $f$.
\textbf{Warning:} this notation should not be confused with the use of $\mg$ in
the context of the full \lpsmap algorithm: consider the remainder of the section
self-contained. Equation~\ref{eqn:subprob_nov} becomes
\begin{equation}\label{eqn:degreeadjqp}
\begin{aligned}
\mathrm{minimize}\quad& \hlf\| \mg - \pr \|^2_2\\
\mathrm{subject\,to}\quad& \bm{D}\mg \in \Mpo \subset [0, 1]^d,
\end{aligned}
\end{equation}
where $\Mpo \coloneqq \left\{ \MM \p \mid \p \in \simplex \right\}$ denotes the
set of local constraints over the binary variables.

For any nonempty convex $\Mpo$, this problem has a unique solution, which we denote by $\mg^\star \eqqcolon
F_\Mpo(\pr)$. We will study several specific cases where we can derive
efficient algorithms for computing $F_\Mpo(\pr)$ and its Jacobian
$\pfrac{F_\Mpo}{\pr}$.

\subsection{Preliminaries}

\subsubsection{Projection onto box constraints}
Consider the projection where there are no additional constraints beyond boolean
variables, \ie $\Mpo = [0, 1]^d$.
The constraint $\bm{D}\mg \in [0, 1]^d$ can be equivalently written
\begin{equation}
\mg \in \mathcal{B} \coloneqq \{
\bm{u} \in \reals^d \mid 0 \leq u_i \leq \delta_i^{-1}\}.
\end{equation}
Consider the more general problem:
\begin{equation}\label{eqn:box}
\begin{aligned}
\mathrm{minimize}\quad& \hlf \| \mg - \pr \|^2_2\\
\mathrm{subject\,to}\quad& \alpha_i \leq \mu_i \leq \beta_i.
\end{aligned}
\end{equation}
Its solution is obtained by
noting that it decomposes into $d$ independent one-dimensional problems
\citep[Section~6.2.4]{parikhboyd}
\begin{equation}
\mu_i^\star
= \clip_{[\alpha_i, \beta_i]}(\eta_i)
= \begin{cases}
\alpha_i, & \eta_i \leq \alpha_i; \\
\eta_i, & \alpha_i < \eta_i < \beta_i; \\
\beta_i, & \eta_i \geq \beta_i.\\
\end{cases}
\end{equation}
The derivative of the solution can be obtained by considering all the cases and
is therefore
\begin{equation}
    \frac{\mathrm{d} \mu^\star_i}{\mathrm{d} \eta_i} =
\begin{cases}
1, & \alpha_i < \mu_i^\star < \beta_i \\
0, & \text{otherwise}.\\
\end{cases}
\end{equation}
The Jacobian of the vector-valued mapping is therefore simply the diagonal
matrix with
$\frac{\mathrm{d} \mu^\star_i}{\mathrm{d} \eta_i}
$ along the diagonal;
\begin{equation}
\pfrac{\mg^\star}{\pr} = \diag(\iv{\alpha_i < \mu_i^\star < \beta_i}).
\end{equation}

\subsubsection{Sifting lemma}
This result allows us to break down an otherwise complicated
inequality-constrained optimization problem into two cases which may be simpler
to solve. This turns out to be the case for many factors over relaxed boolean
variables, since the projection onto the set $\mathcal{B}$ can be done in linear
time.
\begin{lemma}
Consider the constraint convex optimization problem
\begin{equation}\label{eqn:constrainedqp}
\begin{aligned}
\mathrm{minimize}\quad& f(\bm{x}) \\
\mathrm{subject\,to}\quad& \bm{x} \in \mathcal{X} \\
&g(\bm{x}) \leq 0.
\end{aligned}
\end{equation}
where $f, g$ are convex and $\mathcal{X} \subset \reals^d$ is nonempty.
Suppose the problem~\ref{eqn:constrainedqp} is feasible and bounded below.
Consider the set of solutions of the relaxed problem obtained by dropping the
inequality constraint, \ie
$\mathcal{A} = \argmin_{\bm{x} \in \mathcal{X}} f(\bm{x}).$
Then
\begin{enumerate}
\item If some $\tilde{\bm{x}} \in \mathcal{A}$ is feasible for
problem (\ref{eqn:constrainedqp})---\ie, $g(\tilde{\bm{x}}) \leq 0$---then $\tilde{\bm{x}}$
is a solution of problem~(\ref{eqn:constrainedqp}).
\item If for all $\tilde{\bm{x}}\in\mathcal{A}, g(\tilde{\bm{x}})>0$,
then the inequality constraint must be active, \ie,
problem~(\ref{eqn:constrainedqp}) is equivalent to
\begin{equation}
\begin{aligned}
\mathrm{minimize}\quad& f(\bm{x}) \\
\mathrm{subject\,to}\quad& \bm{x} \in \mathcal{X} \\
&g(\bm{x}) = 0.
\end{aligned}
\end{equation}
\end{enumerate}
\end{lemma}
For a proof, see \citep[Lemma 17]{ad3}.

\subsubsection{Singly-constrained bounded quadratic programs}\label{supp:scbqp}
Consider the quadratic program
\begin{equation}\label{eqn:prob_knap_tight}
\begin{aligned}
\mathrm{minimize}\quad& \hlf \| \mg - \pr \|^2_2\\
\mathrm{subject\,to}\quad& \alpha_i \leq \mu_i \leq \beta_i\quad\text{for } i \in
[d] \\
& \sum_{j=1}^d w_j \mu_j = B.
\end{aligned}
\end{equation}
Unlike the box constraints above, this problem is rendered more complicated by
the sum constraint which couples all variables together. An efficient algorithm
can be derived due to the following observation.
\begin{proposition}\citep{Pardalos1990}
Let $\mg$ be a feasible point of (\ref{eqn:prob_knap_tight}). Then, $\mg$
is the global minimum if and only if there exists a scalar $\tau \in \reals$
such that, for all $i \in [d]$,
\begin{equation}
\mu_i(\tau) = \clip_{[\alpha_i, \beta_i]}(w_i \tau + \eta_i).
\end{equation}
\end{proposition}
Proof is provided by \citet{Pardalos1990}.%
\footnote{Our formulation recovers
problem (2) of \citet{Pardalos1990} under
the change of variable
$x_i = \frac{\mu_i - \eta_i}{w_i}$
and choice of constants
$c_i = w_i^2,~
d = B - \left(\sum_{j=1}^d w_i \eta_i\right),~
a_i = \frac{\alpha_i - \eta_i}{w_i},~
b_i = \frac{\beta_i - \eta_i}{w_i}.$}
This proposition reduces the optimization
problem to a one-dimensional search, which can be solved iteratively by
bisection, in $\mathcal{O}(d \log d)$ via sorting, or in $\mathcal{O}(d)$ using
selection \citep[as proposed in][]{Pardalos1990}.
Its sparse Jacobian can be computed efficiently, as shown by the following
original result, resembling the result of \citet{entmax}.
\begin{proposition}\label{prop:scbqp_jacob}
Let $G : \reals^d \rightarrow \reals^d$ denote the solution mapping of
problem~\ref{eqn:prob_knap_tight}, \ie, $\mg^\star = G(\pr)$.
Denote the set $\mathcal{I} =
\{ i \in [d] \mid \mg_i^\star \not\in \{\alpha_i, \beta_i\}$. Then,
\begin{enumerate}
\item $(\bm{J})_{ij} = 0$ whenever $i \not\in \mathcal{I}$ or $j \not\in \mathcal{I}$.
\item Denoting $\bar{\bm{J}}_G$ the restriction of the Jacobian to the rows and
columns in $\mathcal{I}$, $\bar{\bm{J}}_G = \bm{I} -
\frac{\bm{ww}^\top}{\bm{w}^\top\bm{w}}$.
\end{enumerate}
Then, $\bm{J}_G \in \pfrac{G}{\pr}$, \ie, it is a generalized Jacobian.
\end{proposition}
\begin{proof}
If $\mu_i^\star = \alpha_i$ (respectively $\beta_i$), then decreasing (respectively
increasing) $\eta_i$ by any amount does not change the solution, therefore a
subgradient is zero. It remains to consider the support.
Let $\bar{\mg}, \bar{\pr}, \bar{\bm{w}}$ denote the restrictions of those
vectors to the indices in $\mathcal{I}$. The KKT conditions on the
support form a linear system
\begin{equation}
 \begin{bmatrix}
     \bm{I} & \bar{\bm{w}} \\
     \bar{\bm{w}}^\top & 0
 \end{bmatrix}
 \begin{bmatrix}
    \bar{\mg} \\
    \tau
 \end{bmatrix}
 =
 \begin{bmatrix}
     \bar\pr \\ B
 \end{bmatrix}.
\end{equation}
Differentiating \wrt $\bar\pr$ yields
\begin{equation}
 \begin{bmatrix}
     \bm{I} & \bar{\bm{w}} \\
     \bar{\bm{w}}^\top & 0
 \end{bmatrix}
 \begin{bmatrix}
    \bar{\bm{J}}_G \\
    \bm{J}_\tau
 \end{bmatrix}
 =
 \begin{bmatrix}
\bm{I} \\ \bm{0}
 \end{bmatrix}.
\end{equation}
Gaussian elimination readily gives
\begin{equation}
\bar{\bm{J}}_G = \bm{I} - \frac{\bm{ww}^\top}{\bm{w}^\top\bm{w}}.
\end{equation}
\end{proof}

\subsection{Logic factors}
\subsubsection{XOR factor (exactly one of d)}\label{sec:factor_xor}
The exclusive OR (XOR) factor over $d$ boolean variables only accepts
assignments in which exactly one is turned on.
The accepted bit vectors are thus indicator vectors $\bm{e}_1, \dots, \bm{e}_d$,
so the matrix $\MM = \bm{I}$ and the constraint set is $\Mpo_\text{XOR} = \conv \{
\bm{e}_1, \dots, \bm{e}_d\} = \simplex^d = \{ \mg \in [0, 1]^d \mid
\bm{1}^\top\mg = 1 \}$. Rewriting the constraint $\bm{D}\mg \in
\Mpo_\text{XOR}$ more explicity, the optimization problem becomes
\begin{equation}\label{eqn:factor_xor}
\begin{aligned}
\mathrm{minimize}\quad& \hlf \| \mg - \pr \|^2_2\\
\mathrm{subject\,to}\quad& 0 \leq \mu_i \leq \nicefrac{1}{\delta_i} \quad\text{for } i \in
[d] \\
& \sum_{j=1}^d \delta_j \mu_j = 1.
\end{aligned}
\end{equation}
Therefore, we may invoke the algorithm from \S\ref{supp:scbqp}, with
$\alpha_i=0, \beta_i = \nicefrac{1}{\delta_i}, w_i = \delta_i, B=1$.  Note that
when all $\delta_i = 1$ (\eg, if the XOR factor is the only factor in the factor
graph), this recovers the differentiable sparsemax transform \citep{sparsemax},
commonly used in neural networks as a sparse attention mechanism.

\subsubsection{OR factor (at least one of d)}
A logical OR factor over $d$ boolean variables encodes the constraint that at
least one variable is turned on; in other words, it permits all assignments
\emph{except} the one where all variables are off. Such a factor is useful for
encoding existential constraints.
Its constraint set is
$\Mpo_\text{OR} = \conv \big(\{0, 1\}^d - \{ \bm{0} \} \big) =
\{ \mg \in [0, 1]^d \mid  \bm{1}^\top\mg \geq 1 \}$, leading to
\begin{equation}
\begin{aligned}
\mathrm{minimize}\quad& \hlf \| \mg - \pr \|^2_2\\
\mathrm{subject\,to}\quad& 0 \leq \mu_i \leq \nicefrac{1}{\delta_i} \quad\text{for } i \in
[d] \\
& \sum_{j=1}^d \delta_j \mu_j \geq 1.
\end{aligned}
\end{equation}
Using the sifting lemma with set $\mathcal{X} = \{ \mg \in \reals^d \mid 0 \leq
\mu_i \leq \nicefrac{1}{\delta_i} \}$, we reduce this problem to either a
simple clipping operation or the XOR problem (\ref{eqn:factor_xor}), as shown in
Algorithm~\ref{alg:factor_or}. In practice, since we don't need the full
Jacobian but just access to Jacobian-vector products, we just need to store an
indicator of \emph{which branch was taken} as well as the set of indices
$\mathcal{I} = \{i \mid 0 < \mu^\star_i < \nicefrac{1}{\delta_i} \}$.
\begin{algorithm}
\small\setstretch{1.15}
\caption{OR factor: forward and backward pass.\label{alg:factor_or}}
\begin{algorithmic}[1]
\STATE $\tilde\mu_i = \clip_{[0, \delta_i^{-1}]}(\eta_i)$
\Comment{compute solution candidate}
\IF{$\sum_j \delta_j\tilde\mu_j \geq 1$
\Comment{by the sifting lemma, we found the solution}}
\STATE $\mg^\star \leftarrow \tilde\mg$
\STATE $\bm{J} \leftarrow \diag(\iv{0 < \mu^\star_i <\nicefrac{1}{\delta_i}})$
\ELSE
\STATE $\mg^\star \leftarrow F_\text{XOR}(\pr)$ \Comment{from \S\ref{sec:factor_xor}}
\STATE $\bm{J} \leftarrow \bm{J}_{F_\text{XOR}}$ \Comment{from Proposition~\ref{prop:scbqp_jacob}}
\ENDIF
\STATE return $\mg^\star, \bm{J}$
\end{algorithmic}
\end{algorithm}

\subsubsection{Knapsack factor}
The knapsack constraint factor is parameterized by a non-negative cost assigned
to each variable $\bm{w} \in \reals^d_{+}$, and a budget $B \in \reals$.
Its marginal polytope is
\begin{equation}
\Mpo_{\text{K}(\bm{c}, B)} = \left\{ \mg \in [0, 1]^d \mid \bm{c}^\top\mg \leq B \right\}.
\end{equation}

The degree-adjusted quadratic subproblem required in the \lpsmap algorithm can
be written as
\begin{equation}\label{eq:prob_knap}
\begin{aligned}
\mathrm{minimize}\quad& \hlf\| \mg - \pr \|^2_2\\
\mathrm{subject\,to}\quad& 0 \leq \mu_i \leq \nicefrac{1}{\delta_i} \quad \text{for } i \in
[d] \\
& \sum_{j=1}^d c_j \delta_j \mu_j \leq B
\end{aligned}
\end{equation}
We may solve this problem again using the sifting lemma, noting that, when the
inequality constraint is tight, we may
invoke the algorithm from \S\ref{supp:scbqp}, with
$\alpha_i=0, \beta_i = \nicefrac{1}{\delta_i}, w_i = \delta_i c_i, B=B$.
The procedure is specified in Algorithm~\ref{alg:factor_knap}.
\begin{algorithm}
\small\setstretch{1.15}
\caption{Knapsack factor: forward and backward pass.\label{alg:factor_knap}}
\begin{algorithmic}[1]
\STATE $\tilde\mu_i = \clip_{[0, \delta_i^{-1}]}(\eta_i)$
\Comment{compute solution candidate}
\IF{$\sum_j c_j \delta_j \tilde\mu_j \leq B$
\Comment{by the sifting lemma, we found the solution}}
\STATE $\mg^\star \leftarrow \tilde\mg$
\STATE $\bm{J} \leftarrow \diag(\iv{0 < \mu^\star_i <\nicefrac{1}{\delta_i}})$
\ELSE
\STATE $\mg^\star \leftarrow G(\pr)$ \Comment{from \S\ref{supp:scbqp}}
\STATE $\bm{J} \leftarrow \bm{J}_{G}$ \Comment{from Proposition~\ref{prop:scbqp_jacob}}
\ENDIF
\STATE return $\mg^\star, \bm{J}$
\end{algorithmic}
\end{algorithm}

\subsubsection{Budget and at-most-one factors}
A special case of the Knapsack factor is useful when we have a budget over the
total number of variables that can be switched on at the same time. In other
words, we take the budget $B$ to be the maximum allowed number of variables, and
the cost $c_i = 1$ for all $i$, leading to
\begin{equation}\label{eq:prob_budget}
\begin{aligned}
\mathrm{minimize}\quad& \hlf\| \mg - \pr \|^2_2\\
\mathrm{subject\,to}\quad& 0 \leq \mu_i \leq \nicefrac{1}{\delta_i} \quad \text{for } i \in
[d] \\
& \sum_{j=1}^d \delta_j \mu_j \leq B.
\end{aligned}
\end{equation}
Perhaps the most commonly encountered version is when $B=1$, meaning at most one
variable can be active (but keeping all variables off is also a legal solution.)
\begin{equation}\label{eq:prob_at_most_one}
\begin{aligned}
\mathrm{minimize}\quad& \hlf\| \mg - \pr \|^2_2\\
\mathrm{subject\,to}\quad& 0 \leq \mu_i \leq \nicefrac{1}{\delta_i} \quad \text{for } i \in
[d] \\
& \sum_{j=1}^d \delta_j \mu_j \leq 1.
\end{aligned}
\end{equation}

\subsubsection{Logical negation}\label{supp:negation}
The ability to impose logical constraints on \emph{negated} boolean variables
opens up many new possiblities, through algebraic manipulation, \eg, DeMorgan's
laws. For instance, we may obtain a negated \emph{conjuction} factor, since
\begin{equation}
\mathcal{Y}_\text{NAND} = \{ \bm{m} \in \{0, 1\}^d \mid
\neg(m_1 \wedge \dots \wedge m_d) \} =
\{ \bm{m} \in \{0, 1\} \mid \neg m_1 \vee \dots \vee \neg m_d\},
\end{equation}
and so $(m_1, \dots, m_d) \in \mathcal{Y}_\text{NAND}$ is equivalent to
$(\neg m_1, \dots, \neg m_d) \in \mathcal{Y}_\text{OR}$. Similarly,
implication may be written as
\begin{equation}
\mathcal{Y}_\text{IMPLY} = \{ \bm{m} \in \{0, 1\}^d \mid m_1 \wedge \dots \
m_{d-1} \implies m_d \},
\end{equation}
and computed using negations and the OR factor, because
\begin{equation}
(m_1, \dots, m_d) \in \mathcal{Y}_\text{IMPLY} \quad\text{is equivalent to}
\quad
(\neg m_1, \dots, \neg m_{d-1}, m_d) \in \mathcal{Y}_\text{OR}.
\end{equation}

\begin{proposition}\label{prop:negation}
Denote by $F_\Mpo(\pr)$ the solution of the relaxed boolean QP
in \eqnref{eqn:degreeadjqp}.
Consider the set obtained from $\Mpo$ by negating the interpretation of
the $k$\textsuperscript{th} boolean variable in the constraints, \ie
\begin{equation}\label{eqn:negation}
\bm{\nu} \in \Mpo^{\neg k} \iff (\nu_1, \dots, 1 - \nu_k, \dots,
\nu_d) \in \Mpo
\end{equation}

Define the weight-aware transformation $\operatorname{flip}_k(\bm{x}) =
\left(x_1, x_2, \dots, \frac{1}{\delta_k} - x_k, \dots, x_d\right)$. Then, we have
\begin{equation}F_{\Mpo^{\neg k}}(\pr) =
\operatorname{flip}_k(F_\Mpo(
\operatorname{flip}_k(\pr))).\end{equation}
\end{proposition}
\begin{proof}
We are looking for the solution $\bar\mg^\star$ of the ``flipped'' problem
\begin{equation}\label{eqn:logicqpflip}
\begin{aligned}
\mathrm{minimize}\quad& \| \bar\mg - \pr \|^2_2\\
\mathrm{subject\,to}\quad& \bm{D}\bar\mg \in \Mpo^{\neg k}.
\end{aligned}
\end{equation}
Denote $\bar{\bm{\nu}} \coloneqq \bm{D}\bar\mg = (\delta_1 \bar\mu_1, \cdots,
\delta_d
\bar\mu_d)$. Applying \eqnref{eqn:negation} we consider the un-flipped variable
\begin{equation}\label{eqn:chgvar}
\bm{\nu} \coloneqq (\delta_1 \bar\mu_1, \cdots, 1 - \delta_k \bar\mu_k, \cdots,
\delta_d \bar\mu_d) \in \Mpo.
\end{equation}
To go back to the form of \eqref{eqn:degreeadjqp}, we make the change of variable
into $\mg$ such that
$\bm{D}\mg = \bm{\nu}$, \ie
\begin{equation*}
\bar\mg \coloneqq \left(\bar\mu_1, \cdots, \frac{1}{\delta_k} - \bar\mu_k, \cdots,
\bar\mu_d\right) = \operatorname{flip}_k(\mg).
\end{equation*}
The objective value after this change of variable becomes
\begin{equation}
\begin{aligned}
\sum_{j} (\bar\mu_j - \eta_j)^2 &=
\sum_{j\neq k} (\mu_j - \eta_j)^2 + \left( \frac{1}{\delta_k} - \mu_k -
\eta_k \right)^2 \\
&= \sum_{j\neq k} (\mu_j - \eta_j)^2 + \left(\mu_k - \left(\frac{1}{\delta_k} -
\eta_k\right)\right)^2 \\
\end{aligned}
\end{equation}
Under the constraints $\bm{D}\mg \in \Mpo$, this is an instance of
\eqref{eqn:degreeadjqp} with modified potentials $\bar{\pr} =
\operatorname{flip}_k(\pr)$,
thus its minimizer is $\mg^\star = F_{\Mpo}(\operatorname{flip}_k(\pr))$.
Undoing the change of variable from \eqnref{eqn:chgvar} yields $\bar\mg^\star =
\operatorname{flip}_k\big(F_{\Mpo}(\operatorname{flip}_k(\pr))\big)$.
\end{proof}
%\begin{corollary}
%\end{corollary}
\begin{corollary}
The Jacobian of $F_{\Mpo^{\neg k}}$ can be obtained from the Jacobian of
$F_{\Mpo}$ by flipping the sign of the $k$\textsuperscript{th} row and
column, \ie,
\begin{equation}
\pfrac{F_{\Mpo^{\neg k}}}{\pr} =
\bm{L}_k
\pfrac{F_{\Mpo}}{\bar\pr}
\bm{L}_k
\quad\text{where}\quad
\bm{L}_k = \diag(1, \dots, \underbrace{-1}_{k}, \dots, 1).
\end{equation}
\end{corollary}

\subsubsection{OR-with-output factor}
This factor lays the foundation for deterministically defining new binary
variables in a factor graph as a logical function of other variables.
The set of boolean vectors valid according to the OR-with-output factor is
\begin{equation}
    \mathcal{Y}_\text{ORout} = \{ \bm{m} \in \{0, 1\}^d \mid m_d = m_1 \vee m_2 \vee \dots
\vee m_{d-1} \}.
\end{equation}
Its convex hull $\Mpo_\text{ORout} = \conv \mathcal{Y}_\text{ORout}$ can be
shown to be \citep{ad3}
\begin{equation}
\Mpo_\text{ORout} = \left\{ \mg \in [0, 1]^d~\middle|~\sum_{j=1}^{d-1} \mu_j \geq
\mu_d, ~\mu_i \leq \mu_d \text{ for all } i \in [d-1] \right\}.
\end{equation}
This leads to the degree-adjusted QP
\begin{equation}\label{eq:prob_orout}
\begin{aligned}
\mathrm{minimize}\quad& \hlf\| \mg - \pr \|^2_2\\
\mathrm{subject\,to}\quad& 0 \leq \mu_i \leq \nicefrac{1}{\delta_i} \quad \text{for } i \in
[d] \\
& \delta_i \mu_i \leq \delta_d \mu_d \quad \text{for } i \in [d-1], \\
& \sum_{j=1}^{d-1} \delta_j \mu_j \leq \delta_d \mu_d.\
\end{aligned}
\end{equation}
We follow \citet{ad3} and write this as the projection onto the set $\mathcal{A}
= \mathcal{U} \cap \mathcal{A}_2 \cap \mathcal{A}_3$, where the individual sets
are slightly different because of the degree correction:
\begin{align}
\mathcal{U} \coloneqq& [0, \nicefrac{1}{\delta_i}] \times \dots \times [0,
\nicefrac{1}{\delta_d}]\label{eqn:constr_u} \\
\mathcal{A}_1 \coloneqq& \{ \mg \in \reals^d \mid \delta_i \mu_i \leq \delta_d
\mu_d \text{ for } i \in [d-1] \}\label{eqn:constr_a1} \\
\mathcal{A}_2 \coloneqq& \left\{ \mg \in \reals^d~\middle|~\sum_{j=1}^{d-1}
\delta_j \mu_j \leq \delta_d \mu_d \right\}
\end{align}

We may apply the sifting lemma iteratively as such:
\begin{enumerate}
\item Set $\tilde\mg = F_\mathcal{U}(\pr)$. If $\tilde\mg \in \mathcal{A}_1 \cap
\mathcal{A}_2$, then $\mg^\star = \tilde\mg$. Else, if $\tilde\mg \not\in
\mathcal{A}_1$, go to step 2, else (if $\tilde\mg \not\in \mathcal{A}_2$)
go to step 3.
\item Compute $\tilde\mg = F_{\mathcal{U} \cap \mathcal{A}_1}.$ If $\tilde\mg
\in \mathcal{A}_2$, then $\mg^\star = \tilde\mg$, else, go to step 3.
\item  From the sifting lemma, the equality constraint in $\mathcal{A}_2$ must
be tight, so we must solve
\begin{equation}\label{eqn:prob_orout_tight}
\begin{aligned}
\mathrm{minimize}\quad& \hlf\| \mg - \pr \|^2_2\\
\mathrm{subject\,to}\quad& 0 \leq \mu_i \leq \nicefrac{1}{\delta_i} \quad \text{for } i \in
[d] \\
& \delta_i \mu_i \leq \delta_d \mu_d \quad \text{for } i \in [d-1], \\
& \sum_{j=1}^{d-1} \delta_j \mu_j~\textcolor{purple}{\bm{=}}~\delta_d \mu_d.\
\end{aligned}
\end{equation}
\end{enumerate}

Let's start by tackling problem~(\ref{eqn:prob_orout_tight}).
Since the sum inequality is tight, every elementwise inequality becomes
\begin{equation}
\delta_i \mu_i \leq \sum_{j-1}^{d - 1} \delta_j \mu_j \iff
0 \leq \sum_{j \in [d-1] - \{i\}} \delta_j \mu_j
\end{equation}
which is trivially true (since $\delta_j \geq 0$ and $\mu_j \geq 0$) and so
the inequalities in $\mathcal{A}_1$ are redundant.
Next, notice that
\begin{equation}
\sum_{j-1}^{d-1} \delta_j \mu_j = \delta_d \mu_d \iff
\sum_{j-1}^{d-1} \delta_j \mu_j + (1 - \delta_d \mu_d) = 1.
\end{equation}
Therefore, direct application of Proposition~\ref{prop:negation} shows that the
remaining problem,
\begin{equation}
\begin{aligned}
\mathrm{minimize}\quad& \hlf\| \mg - \pr \|^2_2\\
\mathrm{subject\,to}\quad& 0 \leq \mu_i \leq \nicefrac{1}{\delta_i} \quad \text{for } i \in
[d] \\
& \sum_{j=1}^{d-1} \delta_j \mu_j=\delta_d \mu_d,
\end{aligned}
\end{equation}
is equivalent to the XOR problem (\S\ref{sec:factor_xor}) with the last variable
negated.

It remains to show how to project onto the intersection $\mathcal{U} \cap
\mathcal{A}_1$. To this end, we prove the following slight generalization of
\citet[Proposition 19]{ad3}. Furthermore, we provide a more detailed derivation
of the resulting algorithm.
\newcommand{\srt}[1]{{\sigma[#1]}}%
\begin{proposition}
Let $\mathcal{A}_1$ be defined as in
\eqnref{eqn:constr_a1}.
Denote by $\srt{\cdot}$ the permutation that sorts the sequence
$\delta_\srt{j} \mu_\srt{j}$ decreasingly, \ie
\begin{equation}\label{eqn:cone_proj_permutation}
\delta_\srt{1}\eta_\srt{1} \geq \delta_\srt{2}\eta_\srt{2} \geq \dots
\geq \delta_\srt{d-1} \eta_\srt{d-1}.
\end{equation}
For any $\rho \in [d-1]$, define
\begin{align}
\mathcal{S}(\rho) &\coloneqq \{ \srt{1}, \dots, \srt{\rho}\} \cup \{ d \} \\
\tau(\rho) &\coloneqq
\frac{\sum_{j \in \mathcal{S}(\rho)} \nicefrac{\eta_j}{\delta_j}}%
{\sum_{j \in \mathcal{S}(\rho)} \nicefrac{1}{\delta_j^2}}
\end{align}
Let $\bar\rho$ be the smallest $\rho<d-1$ satisfying $\tau(\rho) \geq
\delta_\srt{\rho+1} \eta_\srt{\rho+1}$, or $\rho=d-1$ if none exists.
Then, $F_{\mathcal{A}_1}(\pr)$ is
\begin{equation}
\mu^\star_i = \begin{cases}
\frac{\tau(\bar\rho)}{\delta_i},&i \in \mathcal{S}(\bar\rho); \\
\eta_i,&i \not\in \mathcal{S}(\bar\rho).
\end{cases}
\end{equation}
\end{proposition}
\begin{proof}
The problem we are trying to solve is
\begin{equation}\label{eqn:cone_proj}
\begin{aligned}
\mathrm{minimize}\quad& \hlf\| \mg - \pr \|^2_2\\
\mathrm{subject\,to}\quad& \delta_i\mmg_i \leq \delta_d \mmg_d\quad\text{for } i \in [d-1].
\end{aligned}
\end{equation}
The objective fully decomposes into $d$ subproblems, but they are all coupled
with the last variable $\mu_d$ through the constraints, so we can write the
problem equivalently as
\begin{equation}
\argmin_{\mmg_d \in \RR}\left[
\hlf (\mmg_d - \ppr_d)^2
+ \sum_{j=1}^{d-1} \min_{\delta_j \mmg_j \leq
\delta_d \mmg_d}
\hlf (\mmg_j - \ppr_j)^2
\right],
\end{equation}
or, after making the change of variable
$\tau \coloneqq \delta_d \mmg_d$,
\ie, $\mmg_d = \frac{\tau}{\delta_d}$,
\begin{equation}
\argmin_{\tau \in \RR}\left[
\hlf \left(\frac{\tau}{\delta_d} - \ppr_d\right)^2
+ \sum_{j=1}^{d-1} \min_{\delta_j \mmg_j \leq
\tau}
\hlf (\mmg_j - \ppr_j)^2
\right].
\end{equation}
Consider one of the nested minimizations,
\begin{equation}
\min_{\delta_j \mmg_j \leq \tau}
\hlf(\mmg_j - \ppr_j)^2.
\end{equation}
Ignoring the constraints for a moment, the solution would be $\mmg_j^\star =
\ppr_j$ with an objective value of $0$. If this solution is infeasible, the
constraint must be tight, leading to the two cases:
\begin{equation}\label{eqn:cone_proj_sol_wrt_tau}
\mmg_j^\star = \begin{cases}
    \ppr_j, & \text{if}~
\delta_j \ppr_j \leq \tau, \\
\frac{\tau}{\delta_j}, & \text{otherwise}.\\
\end{cases}
\end{equation}
The contribution of the $j$\textsuperscript{th} term to the objective value is
\begin{equation}
    \hlf(\mmg_j^\star - \eta_j)^2 =
    \begin{cases}
    0, & \text{if}~\delta_j \eta_j \leq \tau, \\
    \hlf \left(\frac{\tau}{\delta_j} - \ppr_j\right)^2,
    & \text{otherwise}.\\
    \end{cases}
\end{equation}
Assume for now that we know upfront the support
$\mathcal{S}^\star \coloneqq \{j : \delta_j \ppr_j > \tau\}
\cup \{ d \}$.
The optimum objective value is
\begin{equation}\label{eqn:cone_proj_objective}
    F(\tau;\pr) = \hlf\left(\frac{\tau}{\delta_d} - \ppr_d\right)^2 +
\sum_{j : \delta_j \ppr_j > \tau} \hlf \left(\frac{\tau}{\delta_j}
- \ppr_j\right)^2
= \sum_{j \in \mathcal{S}^\star}
\hlf \left(\frac{\tau}{\delta_j} - \ppr_j\right)^2,
\end{equation}
so we can solve for $\tau^\star$ given $\mathcal{S}^\star$ by setting the
gradient to zero:
\begin{equation}
    0 \setto F(\tau;\pr) = \sum_{j \in \mathcal{S}^\star}
    \frac{1}{\delta_j} \left( \frac{\tau}{\delta_j} - \eta_j \right)
\end{equation}
which leads to the expression
\begin{equation}\label{eqn:cone_proj_tau_wrt_s}
    \tau^\star =
\left(\sum_{j \in \mathcal{S}^\star} \frac{1}{\delta_j^2}\right)^{-1}
\left(\sum_{j \in \mathcal{S}^\star} \frac{\eta_j}{\delta_j}\right).
\end{equation}
The entire solution $\mg^\star$ minimizing \eqnref{eqn:cone_proj}
is therefore uniquely determined by its $\mathcal{S}^\star$, since the
support lets us identify $\tau^\star$ (\eqnref{eqn:cone_proj_tau_wrt_s})
and the remaining variables are a function of $\tau^\star$
(Equation~\ref{eqn:cone_proj_sol_wrt_tau}).
At a glance, there appear to be exponentially many choices for $\mathcal{S}$.
We next prove a few results that, taken together, simplify this search to a
linear sweep over a sorted set, corresponding to the procedure described in the
proposition.

\paragraph{The possible supports are ordered.} Pick $i, j \in [d-1]$
such that $\delta_i\eta_i \leq \delta_j \eta_j$. If $i \in \mathcal{S}^\star$,
we have $\tau < \delta_i \eta_i \leq \delta_j \eta_j$, therefore
$j\in\mathcal{S}^\star$ as well. Consequently, defining $\sigma$ as in
Equation~\ref{eqn:cone_proj_permutation}, the possible supports are:
\begin{equation}
\mathcal{S}(0) = \{ d \}; \quad \mathcal{S}(1) = \{\sigma[1], d \};
\quad\dots;\quad \mathcal{S}(d-1) = \{ \sigma[1], \sigma[2], \dots, \sigma[d-1], d \} = [d].
\end{equation}

\paragraph{Not all of the $d$ sets above are feasible.} For each $\rho \in \{0,
\dots, d-1 \}$, Equation~\ref{eqn:cone_proj_tau_wrt_s} yields the $\tau(\rho)$ that would be
obtained if $\mathcal{S}(\rho)$ were the true support. But if
$\mathcal{S}(\rho)$ is the true support $\mathcal{S}^\star$, then by definition
$\tau \geq \delta_j \eta_j$ for any $j \not\in \mathcal{S}(\rho)$. If $\rho =
d-1$, $\mathcal{S}(\rho-1)=[d]$ so this is vacuously true. For $\rho < d - 1$ we have to check
that $\tau(\rho) \geq \delta_j \eta_j$ for $j \in \mathcal{S}^C(\rho) = \{\sigma[\rho+1], \dots,
\sigma[d-1] \}$. This is equivalent to checking $\tau(\rho) > \max_{j \in
\mathcal{S}^C(\rho)} \delta_j \eta_j = \delta_{\sigma[\rho+1]} \eta_{\sigma[\rho+1]}$.

\paragraph{Smaller $\mathcal{S}$ are better.} Inspecting the objective value in
Equation~\ref{eqn:cone_proj_objective}, for any $\rho < \rho'$, the difference
$F(\tau(\rho'); \pr) - F(\tau(\rho); \pr) \geq 0$ as a sum of squares.
Therefore, a smaller $\rho$ is always as least as good in terms of objective
value, so the smallest feasible $\rho$ must be optimal, concluding the proof.
\end{proof}

It remains to show that incorporating the box constraints $\mathcal{U}$ can be
done through simple composition. To this end, we will first prove two
observations about the invariance of projections onto $\mathcal{A}_1$.
\begin{corollary}\label{coro:const} Let $\tilde\eta_j \coloneqq \eta_j +
    \frac{c}{\delta_j}$ for a constant $c \in \RR$. We have
    $\tilde\mu^\star_j = \mu^\star_j + \frac{c}{\delta_j}$,
    $\tilde\tau^\star = \tau^\star + c$, and $\tilde{\mathcal{S}}^\star =
    \mathcal{S}^\star$.
\end{corollary}
\begin{proof} For $i, j$, if $\delta_i \eta_i \geq \delta_j \eta_j$, then
    $\delta_i \tilde\eta_i \geq \delta_j \eta_j$, so the permutation $\sigma$
    remains the same. We have
    \begin{equation}
        \tilde\tau(\rho) =
\left(\sum_{j \in \mathcal{S}^\star} \frac{1}{\delta_j^2}\right)^{-1}
\left(\sum_{j \in \mathcal{S}^\star} \frac{\eta_j + \nicefrac{c}{\delta_j}}{\delta_j}
\right)
=
\left(\sum_{j \in \mathcal{S}^\star} \frac{1}{\delta_j^2}\right)^{-1}
\left(\sum_{j \in \mathcal{S}^\star} \frac{\eta_j}{\delta_j} +
c\sum_{j \in \mathcal{S}^\star} \frac{1}{\delta_j^2}\right) =
\tau(\rho) + c.
    \end{equation}
The feasability condition for $\rho$ remains equivalent:
\begin{equation}
\begin{aligned}
    \tilde\tau(\rho) &> \delta_\srt{\rho+1} \tilde\eta_\srt{\rho+1}  \iff \\
        \tau(\rho) + c &> \delta_\srt{\rho+1} \left(\eta_\srt{\rho+1} +
        \frac{c}{\delta_\srt{\rho+1}}\right) = \delta_\srt{\rho+1} \eta_\srt{\rho+1} +
        c.
\end{aligned}
\end{equation}
Therefore, the optimal $\rho$ for $\pr$ is also optimal for $\tilde\pr$. As
$\tilde\tau^\star = \tau^\star+c$, we have $\tilde\mu^\star_j =
\mu^\star_j + \frac{c}{\delta_j}$ for all j.
\end{proof}
\begin{corollary}\label{coro:supp} Let $\mg^\star = F_{\mathcal{A}_1}(\pr)$ with support
    $\mathcal{S}^\star$.
    Define
    \begin{equation}
        \tilde\eta_j \coloneqq \begin{cases}
            \text{any } \tilde\eta_j \leq \frac{\tau}{\delta_j},& j \not\in \mathcal{S}^\star \\
            \eta_j, & j \in \mathcal{S}^\star.
        \end{cases}
    \end{equation}
    Then, $F_{\mathcal{A}_1}(\tilde\pr) \coloneqq \tilde\mg^\star = \mg^\star$.
\end{corollary}
\begin{proof}
    By construction, the permutation $\tilde\sigma$ is constant for the first
    $\rho^\star$ indices. By choice of $\tilde\eta_\srt{\rho+1}$, the
    feasability condition is satisfied, so $\tilde\rho^\star = \tilde\rho$.
    Since $\tilde\tau^\star$ depends only on the unchanged indices, the solution
    is the same.
\end{proof}
With these observations, we may now prove the following decomposition result.
\begin{proposition}
    For any $\pr \in \reals^d, F_{\mathcal{B} \cap \mathcal{A}_1} =
F_\mathcal{B}\left(F_{\mathcal{A}_1}(\pr)\right)$.
\end{proposition}
\begin{proof}
We invoke \citet[Lemma 18]{ad3},
in order to show that Dykstra's algorithm for projecting onto $\mathcal{A}_1
\cap \mathcal{B}$
converges after one iteration. This requires showing
\begin{equation}
    F_{\mathcal{A}_1}(\underbrace{\pr + \mg^\star - \mg'}_{\pr'}) = \mg^\star,
\end{equation}
where $\mg' = F_{\mathcal{A}_1}(\pr)$ and $\mg^\star = F_{\mathcal{B}}(\mg')$.

We have
\begin{equation}
    \mu_j' = \begin{cases}
        \frac{\tau}{\delta_j}, j \in \mathcal{S}^\star \\
        \eta_j, j \not\in \mathcal{S}^\star.
    \end{cases}
\end{equation}
We apply Corollary~\ref{coro:const} with $c = \clip_{[0, 1]}(\tau) - \tau$,
yielding
\begin{equation}\label{eqn:mu_tilde}
    \tilde\mu_j = F_{\mathcal{A}_1}(\tilde{\pr}) =
    \begin{cases}
        \frac{\tau}{\delta_j} + \clip_{[0,
        \delta_j^{-1}]}\left(\frac{\tau}{\delta_j}\right) -
        \frac{\tau}{\delta_j}, &j \in \mathcal{S}^\star \\
        \eta_j + \clip_{[0,
        \delta_j^{-1}]}\left(\frac{\tau}{\delta_j}\right) -
        \frac{\tau}{\delta_j}, &j \not\in \mathcal{S}^\star \\
    \end{cases}
    \quad=
    \begin{cases}
        \clip_{[0, \delta_j^{-1}]}\left(\frac{\tau}{\delta_j}\right),
         &j \in \mathcal{S}^\star \\
        \eta_j + \clip_{[0,
        \delta_j^{-1}]}\left(\frac{\tau}{\delta_j}\right) -
        \frac{\tau}{\delta_j}, &j \not\in \mathcal{S}^\star. \\
    \end{cases}
\end{equation}

Now, observe that
\begin{equation}\label{eqn:eta_prime}
    \eta_j' = \begin{cases}
        \eta_j + \frac{\clip_{[0, 1]}(\tau) - \tau}{\delta_j},& j \in
        \mathcal{S}^\star \\
        \clip_{[0, \delta_j^{-1}]}(\eta_j),& \text{otherwise}.
    \end{cases}
    \quad = \begin{cases}
        \tilde\eta_j, &j \in \mathcal{S}^\star \\
        \clip_{[0, \delta_j^{-1}]}(\eta_j),& \text{otherwise}
    \end{cases}.
\end{equation}
We can now apply Corollary \ref{coro:supp} to show that $\tilde\mg^\star =
F_{\mathcal{A}_1}(\pr')$. This requires showing that
$\clip_{[0, \delta_j^{-1}]}(\eta_j) \leq \frac{\tilde\tau}{\delta_j}$ for $j
\not\in \mathcal{S}^\star$. But the latter implies
\begin{equation}
    \begin{aligned}
        \delta_j \eta_j &\leq \tau & \\
        \iff \quad \clip_{[0, 1]}(\delta_j \eta_j) &\leq \clip_{[0, 1]}(\tau) &
        \text{(clipping is non-decreasing)}\\
        \iff \quad \clip_{[0, 1]}(\delta_j \eta_j) &\leq \tau +
            \underbrace{\clip_{[0, 1]}(\tau) - \tau}_{=c} &\\
        \iff \quad \clip_{[0, 1]}(\delta_j \eta_j) &\leq \tilde\tau&\\
        \iff \quad \frac{\clip_{[0, 1]}(\delta_j \eta_j)}{\delta_j} &\leq
            \frac{\tilde\tau}{\delta_j}&\\
        \iff \quad \clip_{[0, \delta_j^{-1}]}(\eta_j) &\leq
            \frac{\tilde\tau}{\delta_j}&
    \end{aligned}
\end{equation}

Putting together the second branch from Equation~\ref{eqn:eta_prime} with the
first branch from Equation~\ref{eqn:mu_tilde}, we get
\begin{equation}
    \tilde\mu_j^\star = \begin{cases}
        \clip_{[0, \delta_j^{-1}]}\left(\frac{\tau}{\delta_j}\right),
         &j \in \mathcal{S}^\star \\
        \clip_{[0, \delta_j^{-1}]}(\eta_j),& \text{otherwise}
    \end{cases}
    \quad=\clip_{[0, \delta_j^{-1}]}(\mu'_j) = \mu^\star_j.
\end{equation}
\end{proof}

\paragraph{Gradient computation} The Jacobian of $F_\text{ORout}$ depends on
which branch was taken. If taking the first branch (\ie, the clipping solution
was feasible), it is simply the Jacobian of clipping,
 $\bm{J}_\text{ORout} = \diag(\iv{0 < \delta_i \mmg_j < 1})$.
If taking the third branch, it is the XOR Jacobian with the last variable
negated, \ie
 $\bm{J}_\text{ORout} = \bm{L}_{d} \bm{J}_\text{XOR} \bm{L}_d$. Otherwise, if
taking the second branch, $\mg^\star = F_\mathcal{B}(F_{\mathcal{A}_1}(\pr))$
and we must work out the Jacobian of $F_{\mathcal{A}_1}$.
Recall that $\mg^\star = F_{\mathcal{A}_1}(\pr)$ has the expression
\begin{equation}
\mmg^\star_j = \begin{cases}\ppr_j,& j \not\in \mathcal{S} \\
\nicefrac{\tau}{\delta_j},& j \in \mathcal{S}.\end{cases}
\end{equation}
For indices $j \not\in \mathcal{S}$, we then have the $j$\textsuperscript{th}
row $\pfrac{\mmg_j}{\pr} =
\bm{e}_j$.
For $j \in \mathcal{S}$, $\pfrac{\mmg_j}{\pr} = \diag(\bm{\delta})^{-1}
\pfrac{\tau}{\pr}$.
Differentiating $\tau^\star$ from Equation~\ref{eqn:cone_proj_tau_wrt_s}
gives
\begin{equation}
\pfrac{\tau}{\ppr_i} = \begin{cases}
0 & i \not \in \mathcal{S} \\
\left(\sum_{k \in \mathcal{S}^\star} \frac{1}{\delta_k^2}\right)^{-1}
\frac{1}{\delta_i} & i \in \mathcal{S}.
\end{cases}
\quad\text{so}\quad
\pfrac{\mmg_j}{\ppr_i} = \begin{cases}
0 & i \not \in \mathcal{S} \\
\left(\sum_{k \in \mathcal{S}^\star} \frac{1}{\delta_k^2}\right)^{-1}
\frac{1}{\delta_i\delta_j} & i \in \mathcal{S}.
\end{cases}
\end{equation}
Combining the cases and applying the chain rule gives the Jacobian for this
branch, which is rank-1 plus diagonal.

\subsection{Pairwise factors for Ising models}
The pairwise factor is a fundamental building block in factor graphs, allowing
to capture soft correlations between two binary variables.

\subsubsection{Deriving the marginal polytope}
In a naive, fully explicit parametrization, we would have two scores for each
binary variable (one for each state), and four scores for every joint
assignment. In this section, however, we show how to reduce this parametrization
to a problem with only three variables $\mmg_1, \mmg_2,$ and $\mmg_{12}$.
Denoting the binary variable states as $F$ and $T$, we have
\begin{equation}
\bm{D}\mg_M =
\begin{bmatrix}
\delta_1(\mg_M)_{1,F} \\
\delta_1(\mg_M)_{1,T} \\
\delta_2(\mg_M)_{2,F} \\
\delta_2(\mg_M)_{2,T} \\
\end{bmatrix}
=
\begin{bmatrix}
1 & 1 & 0 & 0 \\
0 & 0 & 1 & 1 \\
1 & 0 & 1 & 0 \\
0 & 1 & 0 & 1 \\
\end{bmatrix}
\p
\quad\text{and}\quad
\mg_N = \bm{I}\p.
\end{equation}
each element of $\bm{D}\mg_M$ and $\mg_N$
is a sum of elements of $\p$, hence non-negative.
Write $\p = (p_{FF}, p_{FT}, p_{TF}, p_{TT})$ corresponding to the four possible
joint assignments, and observe that
\begin{equation}
\delta_1 \left((\mg_M)_{1,F} + (\mg_M)_{1,T}\right)
= (p_{FF} + p_{FT}) + (p_{TF} + p_{TT}) = 1,
\end{equation}
and similarly $\delta_2 \left((\mg_M)_{2,F} + (\mg_M)_{2,T}\right) = 1.$
We may thus write, for simplicity
\begin{equation}
\mg_M = \left(\nicefrac{1}{\delta_1} - \mmg_1, \mmg_1,
\nicefrac{1}{\delta_2}-\mmg_2, \mmg_2\right)
\quad\text{such that}\quad
\bm{D}\mg_M = \left(1 - \delta_1 \mmg_1, \delta_1 \mmg_1, 1 - \delta_2 \mmg_2,
\mmg_2\right).
\end{equation}
Denote $p_{TT} \eqqcolon \mu_{12}$; we may eliminate $\p$ as:
\begin{equation}\label{eqn:pair_elim_p}
\begin{aligned}
p_{TF} &= \delta_1\mu_1 - \mu_{12}, \\
p_{FT} &= \delta_2\mu_2 - \mu_{12}, \\
p_{FF} &= 1 + \mu_{12} - \delta_1\mu_1 - \delta_2\mu_2.
\end{aligned}
\end{equation}
Considering $\p \geq 0$, this gives the constraints on $\mg$:
\begin{equation}
\begin{aligned}
\delta_1\mu_1 &\geq \mu_{12}, \\
\delta_1\mu_2 &\geq \mu_{12}, \\
\mu_{12} &\geq \delta_1\mu_1 + \delta_2\mu_2 - 1.
\end{aligned}
\end{equation}
In addition, we have the inherited constraints from the definition of $\mg$:
\begin{equation}
\begin{aligned}
0 \leq \delta_1\mu_1 &\leq 1 \\
0 \leq \delta_1\mu_2 &\leq 1 \\
0 \leq \mu_{12} &\leq 1
\end{aligned}
\end{equation}

Therefore, the standard pairwise factor may be reparametrized using the following
constraint set $(\delta_1 = \delta_2=1)$:
\begin{equation}
\Mpo_\text{pair} = \left\{ \mg \in \RR^3_+ \mid
\mmg_{12} \leq \mmg_1 \leq 1;
\mmg_{12} \leq \mmg_2 \leq 1;
\mmg_1 + \mmg_2 - 1 \leq \mmg_{12} \right\}.
\end{equation}
and the constraint set for the degree-adjusted QP is
\begin{equation}
\tilde\Mpo_\text{pair} = \left\{ \mg \in \RR^3_+ \mid
\mmg_{12} \leq \delta_1\mmg_1 \leq 1;
\mmg_{12}\leq \delta_2\mmg_2 \leq 1;
\delta_1 \mmg_1 + \delta_2 \mmg_2 - 1 \leq \mmg_{12}
\right\}.
\end{equation}

Assume we are given $[\pr_M; \pr_N]$, how to convert them to $(\ppr_1, \ppr_2,
\ppr_3)$ such that the solution to the degree-adjusted QP is the same? To answer
this, we compute the objective value as a function of $(\mmg_1, \mmg_2,
\mmg_{12})$. The objective is $\DP{\pr_M}{\mg_M} + \DP{\pr_N}{\mg_N}
-\frac{1}{2} \| \mg_M \|^2$. Substituting $\mg_M$, the first term is
\begin{equation}
\begin{aligned}
\DP{\pr_M}{\mg_M}
&= (\pr_M)_{1,F} \left(\frac{1}{\delta_1} - \mmg_1 \right)
+ (\pr_M)_{1,T} \mmg_1
+ (\pr_M)_{2,F} \left(\frac{1}{\delta_2} - \mmg_2 \right)
+ (\pr_M)_{2,T} \mmg_2 \\
&=
\big((\pr_M)_{1,T} - (\pr_M)_{1,F} \big) \mmg_1
+ \big((\pr_M)_{2,T} - (\pr_M)_{2,F} \big) \mmg_2
+ \text{const.}
\end{aligned}
\end{equation}

The regularizer becomes
\begin{equation}
\begin{aligned}
\frac{1}{2} \| \mg \|^2 &=
\frac{1}{2} \Bigg(
\left(\frac{1}{\delta_1} - \mmg_1 \right)^2
+ \mmg_1^2
+ \left(\frac{1}{\delta_2} - \mmg_2 \right)^2
+ \mmg_2^2
\Bigg) \\
&= \mmg_1^2 + \mmg_2^2 + \frac{\mmg_1}{\delta_1} + \frac{\mmg_2}{\delta_2}
+ \text{const.}
\end{aligned}
\end{equation}

Noting that $\mg_N = \p$ and using Equation~\ref{eqn:pair_elim_p}, the second
term becomes
\begin{equation}
\begin{aligned}
    \DP{\pr_N}{\mg_N} &=
(\pr_N)_{FF} ( 1 + \mmg_{12} - \delta_1 \mmg_1 - \delta_2 \mmg_2) +
(\pr_N)_{TF} (\delta_1 \mmg_1 - \mmg_{12}) +
(\pr_N)_{FT} (\delta_2 \mmg_2 - \mmg_{12}) +
(\pr_N)_{TT} (\mmg_{12}) \\
&= (\delta_1 (\pr_N)_{TF} - \delta_1 (\pr_N)_{FF}) \mmg_1
+ (\delta_2 (\pr_N)_{FT} - \delta_2 (\pr_N)_{FF}) \mmg_2
\\&+ ((\pr_N)_{FF} - (\pr_N)_{TF} - (\pr_N)_{FT} + (\pr_N)_{TT}) \mmg_{12}
+ \text{const.}
\end{aligned}
\end{equation}
Adding all terms leads to a polynomial with coefficients $1$ for $\mmg_1$ and
$\mmg_2$. Scaling by 2 and identifying the coefficients to align with $\ppr_1
\mmg_1 + \ppr_2 \mmg_2 + \ppr_{12} \mmg_{12} - \frac{1}{2} \big( \mmg_1^2 +
\mmg_2^2 \big)$ yields the answer:
\begin{equation}
\begin{aligned}
\ppr_{1} &= \hlf\Big(
(\pr_M)_{1,T} - (\pr_M)_{1,F} + \nicefrac{1}{\delta_1}
+ \delta_1 \big((\pr_N)_{TF} - (\pr_N)_{FF} \big) \Big) \\
\ppr_{2} &= \hlf\Big(
(\pr_M)_{2,T} - (\pr_M)_{2,F} + \nicefrac{1}{\delta_2}
+ \delta_2 \big((\pr_N)_{FT} - (\pr_N)_{FF} \big) \Big) \\
\ppr_{12} &= \hlf\big(
 (\pr_N)_{FF}
-(\pr_N)_{FT}
-(\pr_N)_{TF}
+(\pr_N)_{TT} \big).
\end{aligned}
\end{equation}
\subsubsection{Closed-form solution}
The optimization problem we tackle is
\begin{align}
\mathrm{minimize}\quad& \hlf (\ppr_1 - \mmg_1)^2 + \hlf(\ppr_2 - \mmg_2)^2 -
\ppr_{12} \mmg_{12} \\
\mathrm{subject\,to}\quad&
0 \leq \delta_1\mmg_1 \leq 1; \quad
0 \leq \delta_2\mmg_2 \leq 1; \quad
0 \leq \mmg_{12}; \label{eqn:pair_box}\\
&\delta_1\mmg_1 \geq \mmg_{12}; \quad
\delta_2\mmg_2 \geq \mmg_{12};  \label{eqn:pair_lower_12}\\
& \mmg_{12} \geq \delta_1 \mmg_1 + \delta_2 \mmg_2 - 1 \label{eqn:pair_conj_bound}.
\end{align}

If $\ppr_{12} < 0$, we can make a change of variable to obtain an equivalent
problem with $\ppr_{12} \geq 0$: set $\mmg'_1 = \mmg_1, \mmg'_2 = \frac{1}{\delta_2} -
\mmg_2$ and $\mmg'_{12} = \delta_1 \mmg_1 - \mmg_{12}$; we can show that
wherever $\mg'$ is feasible so is $\mg$ by inspecting the constraints.
The box constraints on $\mmg'_1$ are unchanged, and on $\mmg'_2$ they are
simply flipped. The constraint $0 \leq \mmg'_{12}$ is equivalent to $\delta_1 \mmg_1 \geq
\mmg_{12}$. The constraint $\delta_1\mmg'_1 \geq \mmg'_{12}$ yields
$\mmg_{12} \geq 0$.
The constraint $\delta_2\mmg'_2 \geq \mmg'_{12}$ becomes
$\delta_2 (\delta_2^{-1} - \mmg_2) \geq \delta_1 \mmg_1 - \mmg_{12}$,
equivalent to the final constraint. And finally,
$\mmg'_{12} \geq \delta_1 \mmg'_1 + \delta_2 \mmg'_2 - 1$ is equivalent to
$\mmg_{12} \leq \delta_2 \mmg_2$. The feasible set is thus preserved by this
change of variable. Setting $\ppr'_1 = \ppr_1 + \delta_1\ppr_{12}, \ppr'_2 =
\frac{1}{\delta_2} - \ppr_2,$ and $
\ppr'_{12} = -\ppr_{12}$, we reach an equivalent problem (same objective value
and constraints)  from which we can easily recover the original solution.

We can thus focus on the case $\ppr_{12} \geq 0$.

Note that the objective is linear in $\mmg_{12}$ so
the largest feasible $\mmg_{12}$ is optimal.  This value can be shown to be:
\begin{equation}
\mmg_{12} = \min(\delta_1 \mmg_1, \delta_2 \mmg_2)
\end{equation}
Indeed, any larger one would violate at least one
constraint in Equation~\ref{eqn:pair_lower_12}. As the minimum of two
non-negative numbers, it is non-negative itself, and we can show that
it satisfies Equation~\ref{eqn:pair_conj_bound} by assuming $\delta_1 \mmg_1
\geq \delta_2 \mmg_2$, so $\mmg_{12} = \delta_2 \mmg_2$. Plugging into the
constraint yields $1 \geq \delta_1 \mmg_1$, which is
true under the upper bound in Equation~\ref{eqn:pair_box}.  (The other case is
also verified, by symmetry.)

Therefore, the lower bounds on
$\mmg_{12}$ are always inactive, and we are left with:
\begin{equation}\label{eqn:pair_simplified}
\begin{aligned}
\mathrm{minimize}\quad& \hlf (\ppr_1 - \mmg_1)^2 + \hlf(\ppr_2 - \mmg_2)^2 -
\ppr_{12} \mmg_{12} \\
\mathrm{subject\,to}\quad&
0\leq \delta_1\mmg_1 \leq 1;
\quad 0 \leq \delta_2\mmg_2 \leq 1\\
&\delta_1\mmg_1 \geq \mmg_{12}; \quad
\delta_2\mmg_2 \geq \mmg_{12}; \\
\end{aligned}
\end{equation}

\begin{proposition}\label{prop:pair}%
The problem  in Equation~\ref{eqn:pair_simplified} with
$\ppr_{12} \geq 0$ has the solution:
\begin{equation*}
\setlength{\arraycolsep}{5pt}
\left\{\!\!\begin{array}{l l l}
\textcolor{gray}{(\mmg_1=)} &
\textcolor{gray}{(\mmg_2=)} & \\
\clip_{[0, \delta_1^{-1}]}(\ppr_1),   &
\clip_{[0, \delta_2^{-1}]}(\ppr_2 + \delta_2 \ppr_{12}),  &
\text{if }\delta_1 \ppr_1 > \delta_2 \ppr_2 + \delta_2^2 \ppr_{12};\\
\clip_{[0, \delta_1^{-1}]}(\ppr_1 + \delta_1 \ppr_{12}), &
\clip_{[0, \delta_2^{-1}]}(\ppr_2), &
\text{if }\delta_2 \ppr_2 > \delta_1 \ppr_1 + \delta_1^2 \ppr_{12}; \\
\displaystyle
\clip_{[0, 1]}\left(
\frac{\delta_1\delta_2^2 \ppr_1 + \delta_1^2\delta_2 \ppr_2 +
\delta_1^2\delta_2^2\ppr_{12}}%
{\delta_1^2 + \delta_2^2}
\right)\nicefrac{}{\delta_1}, &
\displaystyle
\clip_{[0, 1]}\left(
\frac{\delta_1\delta_2^2 \ppr_1 + \delta_1^2\delta_2 \ppr_2 +
\delta_1^2\delta_2^2\ppr_{12}}%
{\delta_1^2 + \delta_2^2}
\right)\nicefrac{}{\delta_2},&\text{otherwise.}
\end{array}\right.
\end{equation*}
\end{proposition}
\begin{proof}
If $\ppr_{12} = 0$, the problem separates and we get
$\mmg^\star_1 = \clip_{[0, \delta_1^{-1}]}(\ppr_1)$ and
$\mmg^\star_2 = \clip_{[0, \delta_2^{-1}]}(\ppr_2)$.

The Lagrangian is
\begin{equation}
\begin{aligned}
    L(\mg, \bm{\alpha}, \lbd, \bm{\nu}) =&
\hlf (\mmg_1 - \ppr_1) ^2 + \hlf (\mmg_2 - \ppr_2)^2 - \mmg_{12} \ppr_{12}
+ \alpha_1 (\mmg_{12} - \delta_1\mmg_1)
+ \alpha_2 (\mmg_{12} - \delta_2\mmg_2) \\
&- \lambda_1 \mmg_1
- \lambda_2 \mmg_2
+ \nu_1 (\delta_1\mmg_1 - 1)
+ \nu_2 (\delta_2\mmg_2 - 1)
\end{aligned}
\end{equation}
and the KKT conditions are:
\begin{align}
(\nabla_{\mmg_i} \mathcal{L} \setto 0) &&
\mmg_i    &= \ppr_i + \delta_i \alpha_i + \lambda_i - \delta_i \nu_i &
i \in \{1, 2\}\\
(\nabla_{\mmg_{12}} \mathcal{L} \setto 0)&&
\alpha_1 + \alpha_2 &= \ppr_{12}  & \\
\text{(complementary slackness)} &&
\lambda_1 \mmg_1 &= 0 & i\in\{1,2\}  \\
&&
\alpha_i (\mmg_{12} - \delta_i\mmg_i) &= 0 & i \in\{1,2\}\\
&&
\nu_i (\delta_i\mmg_i - 1) &= 0   & i \in \{1,2\}\\
\text{(primal feas.)}&&
\mmg_{12} &\leq \delta_i \mmg_i   & i \in \{1,2\}\\
&&
0 &\leq \delta_i \mmg_i \leq 1    & i \in \{1,2\} \\
\text{(dual feas.)} &&
\bm{\alpha},\lbd,\bm{\nu} &\geq 0 &
\end{align}
We consider three cases.
\begin{enumerate}
\item $\delta_1 \mmg_1 > \delta_2 \mmg_2$.

Considering the slacknesses gives
\begin{align}
\delta_1 \mmg_1 > 0 &\implies \lambda_1 = 0; \\
\delta_2 \mmg_2 < 1 &\implies \nu_2 = 0; \\
\mmg_{12} = \delta_2\mmg_2 < \delta_1 \mmg_1 &\implies \alpha_1 = 0 \implies
\alpha_2 = \ppr_{12}.
\end{align}
Plugging into the first two conditions gives
\begin{equation}
\mmg_1 = \ppr_1 - \delta_1 \nu_1; \qquad
\mmg_2 = \ppr_2 + \delta_2 \ppr_{12} + \lambda_2.
\end{equation}
Note that $\nu_1, \lambda_2 \geq 0$, so $\mmg_1 \leq \ppr_1$ and
$\mmg_2 \geq \ppr_2 + \delta_2 \ppr_{12}$.
Were it the case that $\delta_1 \ppr_1 \leq \delta_2 \ppr_2 + \delta_2^2
\ppr_{12}$, we'd have
\begin{equation}
    \delta_1 \mu_1 \leq \delta_1 \ppr_1 \leq \delta_2 \ppr_2 + \delta_2^2
\ppr_{12} \leq \delta_2 \mmg_2
\end{equation}
which contradicts our assumption. Therefore, we must have
\begin{equation}
\delta_1 \ppr_1 > \delta_2 \ppr_2 + \delta_2^2.
\end{equation}
If $\mmg_1 < \frac{1}{\delta_1}$ then $\nu_1 = 0$,  and if $\mmg_2 > 0$ then
$\lambda_2 = 0$. Thus the solution has the form
\begin{equation}
\mmg_1 =
\clip_{[0, \delta_1^{-1}]}(\ppr_1), \qquad
\mmg_2 = \max(0, \ppr_2 + \delta_2 \ppr_{12}).
\end{equation}
\item $\delta_1 \mmg_1 < \delta_2 \mmg_2$.

By symmetry to case 1, we must have
\begin{equation}
\delta_2 \ppr_2 > \delta_1 \ppr_1 + \delta_1^2
\end{equation}
and the solution
\begin{equation}
\mmg_1 = \max(0, \ppr_1 + \delta_1 \ppr_{12}), \qquad
\mmg_2 = \clip_{[0, \delta_2^{-1}]}(\ppr_2).
\end{equation}
\item $\delta_1 \mmg_1 = \delta_2 \mmg_2$.

In this case, $\mmg_{12} = \delta_1 \mmg_1 = \delta_2 \mmg_2$ and the problem
reduces to
\begin{equation}
\begin{aligned}
\mathrm{minimize}\quad& \hlf \left(\frac{\mmg_{12}}{\delta_1} - \ppr_1\right)^2 +
\hlf\left(\frac{\mmg_{12}}{\delta_2} - \ppr_2\right)^2 -
\ppr_{12} \mmg_{12} \\
\mathrm{subject\,to}\quad&  0\leq \mmg_{12} \leq 1.
\end{aligned}
\end{equation}
Setting the gradient to 0 yields
\begin{equation}
    \frac{\mmg_{12}}{\delta_1^2} -
    \frac{\ppr_{1}}{\delta_1} +
    \frac{\mmg_{12}}{\delta_2^2} -
    \frac{\ppr_{2}}{\delta_2} - \ppr_{12} = 0
\end{equation}
leading to the solution
\begin{equation}
\mmg_{12} = \clip_{[0, 1]}\Bigg[ \left(\frac{1}{\delta_1^2} +
\frac{1}{\delta_2^2}\right)^{-1} \left( \frac{\ppr_1}{\delta_1} +
\frac{\ppr_2}{\delta_2} + \ppr_{12} \right) \Bigg].
\end{equation}
which, after some manipulation, takes the desired form.
\end{enumerate}
\end{proof}
\subsubsection{Gradient computation}
The Jacobian of this projection is rather straightforward, albeit involving a
lot of branching. Denoting by $\bm{J}_\text{pair} \coloneqq
\pfrac{F_\text{pair}}{\pr}$, if $\ppr_{12} \geq 0$ we can differentiate the
expressions in Proposition~\ref{prop:pair} to get:
\begin{equation}
\bm{J}_\text{pair} = \begin{cases}
\diag(\iv{0 < \delta_i \mu_i < 1}) \cdot
\begin{bmatrix} 1 & 0 & 0 \\
0 & 1 & \delta_2 \\
\end{bmatrix},& \delta_1 \mmg_1 > \delta_2 \mmg_2 \\
\diag(\iv{0 < \delta_i \mu_i < 1}) \cdot
\begin{bmatrix}
1 & 0 & \delta_1 \\
0 & 1 & 0 \\
\end{bmatrix},& \delta_1 \mmg_1 < \delta_2 \mmg_2 \\
\frac{\iv{0 < \mu_{12} < 1}}{\delta_1^2 + \delta_2^2}
\begin{bmatrix}
\delta_2^2 & \delta_1\delta_2 & \delta_1 \delta_2^2 \\
\delta_1 \delta_2 & \delta_1^2 & \delta_1^2 \delta_2 \\
\end{bmatrix},& \delta_1 \mmg_1 = \delta_2 \mmg_2 \\
\end{cases}
\qquad\qquad\text{(if } \ppr_{12} \geq 0\text{)}
\end{equation}

If $\ppr_{12} < 0$, we must make a change of variable. We construct the modified potentials
$\pr' = (\ppr_1 + \delta_1 \ppr_{12}, \nicefrac{1}{\delta_2} - \ppr_2,
-\ppr_{12})$. This transformation has Jacobian
\begin{equation}
    \pfrac{\pr'}{\pr} = \begin{bmatrix}
1 & 0 & \delta_1 \\
0 & -1 & 0 \\
0 & 0 & -1 \\
\end{bmatrix}
\end{equation}
Then, we solve \wrt $\mg'$ defined as
$\mg' = (\mmg_1, \delta_2^{-1} - \mmg_2, \delta_1 \mmg_1 -\mmg_{12})$. We
discard $\mmg'_{12}$ and map
back to a solution to the original problem with $\mg = (\mmg'_1,
\nicefrac{1}{\delta_2} - \mmg_2')$, giving
\begin{equation}
    \pfrac{\mg}{\mg'} = \begin{bmatrix}
1 & 0 \\
0 & -1 \\
\end{bmatrix}
\end{equation}
Therefore, applying the chain rule, we have
\begin{equation}
\bm{J}_\text{pair} =
\pfrac{\mg}{\mg'}
\pfrac{F_\text{pair}}{\pr'}
\pfrac{\pr'}{\pr}
\end{equation}
which, after evaluating and commuting, gives the expression
(branching using the intermediate solution $\mg'$):
\begin{equation}
\bm{J}_\text{pair} = \begin{cases}
\diag(\iv{0 < \delta_i \mu_i' < 1}) \cdot
\begin{bmatrix} 1 & 0 & \delta_1 \\
0 & 1 & \delta_2 \\
\end{bmatrix},& \delta_1 \mmg'_1 > \delta_2 \mmg'_2 \\
\diag(\iv{0 < \delta_i \mmg'_i < 1}) \cdot
\begin{bmatrix}
1 & 0 & 0 \\
0 & 1 & 0 \\
\end{bmatrix},& \delta_1 \mmg'_1 < \delta_2 \mmg'_2 \\
\frac{\iv{0 < \mmg'_{12} < 1}}{\delta_1^2 + \delta_2^2}
\begin{bmatrix}
\delta_2^2 & -\delta_1\delta_2 & 0 \\
-\delta_1 \delta_2 & \delta_1^2 & 0 \\
\end{bmatrix},& \delta_1 \mmg'_1 = \delta_2 \mmg'_2 \\
\end{cases}
\qquad\qquad\text{(if } \ppr_{12} < 0\text{)}
\end{equation}

\section{Experimental details} \label{supp:experimental}
\subsection{Computing infrastructure}
Our infrastructure consists of 4 machines with the specifications shown in
Table~\ref{table:computing_infrastructure}. The machines were used
interchangeably, and all experiments were executed in a single GPU.
We did not observe large differences in the execution time of our models across different machines.
 Furthermore, all of our models fit in a single GPU.

\begin{table}[!htb]
    \small
    \begin{center}
    \begin{tabular}{l ll}
        \toprule
        \sc \# & \sc GPU & \sc CPU  \\
        \midrule
        % hermes
        1.   & 4 $\times$ Titan Xp - 12GB           & 16 $\times$ AMD Ryzen 1950X @ 3.40GHz - 128GB \\
        % athena
        2.   & 4 $\times$ GTX 1080 Ti - 12GB        & 8 $\times$ Intel i7-9800X @ 3.80GHz - 128GB \\
        % zeus
        3.   & 3 $\times$ RTX 2080 Ti - 12GB        & 12 $\times$ AMD Ryzen 2920X @ 3.50GHz - 128GB \\
        % hera
        4.   & 3 $\times$ RTX 2080 Ti - 12GB        & 12 $\times$ AMD Ryzen 2920X @ 3.50GHz - 128GB \\
        \bottomrule
    \end{tabular}
    \end{center}
    \caption{Computing infrastructure.}
    \label{table:computing_infrastructure}
\end{table}
\subsection{ListOps}
\paragraph{Dataset.} Starting with the ListOps dataset,
following \citet{caio-acl}
we convert the constituent structures to dependency trees and remove the
sequences longer than 100 tokens. We put aside a subset of the training data for
validation purposes, leading to a train/validation/test split of
70446/10000/8933 sequences.

\paragraph{Network and optimization settings.}
We use an embedding size and hidden layer size of 50.
The BiLSTM uses a hidden and output size of 25 (so that its
concatenated output has dimension 50).
Like \citet{caio-acl}, we optimize using Adam
with a learning rate of 0.0001. We use a batch size of 64 and no dropout. We
monitor tagging $F_1$ score on the validation set and decay the learning rate by
a factor of .9 when there is no improvement.

\paragraph{\lpsmap settings.} For the \smap baseline, we perform 10
iterations of the active set method. For \lpsmap, we use $\gamma=0.5$, perform
10 outer ADMM iterations, and 10 inner active set iterations,
warm-started from the previous solution.
We use a primal and dual convergence criterion of $\epsilon_p = \epsilon_d =
10^{-6}$. In the backward pass, we perform 100 power iterations.

\subsection{Natural Language Inference}
\paragraph{Network and optimization settings.}
We use 300-dimensional GloVe embeddings, kept frozen (not updated during
training.) We use a dimension of 100 for all other hidden layers, and ReLU
non-linearities.
We use a batch size of 128, dropout of .33, and tune the Adam learning rate
among $0.001 \cdot 2^k$ for $ k \in \{ -3, -2, -1, 0, 1 \}$.

\paragraph{\lpsmap settings.} We use exactly the same configuration as for the
ListOps task above.

\subsection{Multilabel}
\paragraph{Datasets.}
The \emph{bibtex} dataset comes with a given test split. For the
\emph{bookmarks} dataset we leave out a random test set. The dimensions and
statistics of the data are reported in Table~\ref{tab:mldata}.
\begin{table}
\centering
\caption{Multilabel dataset statistics.\label{tab:mldata}}
\begin{tabular}{l r r r r r r r}
\toprule
& samples & train & test & features & labels & density & cardinality \\
\midrule
\emph{bibtex} & 7395 & 4880 & 2515 & 1836 & 159 & 0.015 & 2.402 \\
\emph{bookmarks} & 87856 & 70284 & 17572 & 2150 & 208 & 0.010 & 2.028 \\
\bottomrule
\end{tabular}
\end{table}

\paragraph{Network and optimization settings.}
We use two 300-dimensional hidden layers with ReLU non-linearities..
We use a batch size of 32, no dropout, and an Adam learning rate of 0.001.

\paragraph{\lpsmap settings.}
For both LP-MAP and \lpsmap, we employ the same ADMM optimization settings.
For \emph{bibtex}, we use 100 iterations of ADMM, while for the larger
\emph{bookmarks} we use only 10. We use $\gamma=0.1$, the default value in \adq.
We use a primal and dual convergence criterion of $\epsilon_p = \epsilon_d =
10^{-6}$.
(As pairwise factors have closed-form solutions, the active set algorithm is not
used.)

\section{Code Samples}\label{supp:code}
We include here some self-contrained example scripts demonstrating the use of
\lpsmap for two of the models used in this paper. Up-to-date versions of these
scripts are available at \url{https://github.com/deep-spin/lp-sparsemap/tree/master/examples}.
\begin{listing}[h]
\begin{minted}{python}
import torch
from lpsmap import TorchFactorGraph, Xor, AtMostOne

def main():
    m, n = 3, 5

    eta = torch.randn(m, n, requires_grad=True)

    fg = TorchFactorGraph()
    u = fg.variable_from(eta)

    for i in range(m):
        fg.add(Xor(u[i, :]))

    for j in range(n):
        fg.add(AtMostOne(u[:, j]))  # some columns may be 0

    fg.solve()
    print(u.value)

    u.value[0, -1].backward()
    print(x.grad)

if __name__ == '__main__':
    main()
\end{minted}
\caption{Linear assignment problem using \lpsmap with fine-grained constraints.
(Figure~\ref{fig:match} right).}
\end{listing}

\begin{listing}[h]
\begin{minted}{python}
import torch
from lpsmap import TorchFactorGraph, DepTree, Budget

def main(n=5, constrain=False):

    print(f"n={n}, constrain={constrain}")

    torch.manual_seed(4)

    x = torch.randn(n, n, requires_grad=True)

    fg = TorchFactorGraph()
    u = fg.variable_from(x)
    fg.add(DepTree(u, packed=True, projective=True))

    if constrain:
        for k in range(n):

            # don't constrain the diagonal (root arc)
            ix = list(range(k)) + list(range(k + 1, n))

            fg.add(Budget(u[ix, k], budget=2))

    fg.solve()
    print(u.value)

    u.value[1, -1].backward()
    print(x.grad)

if __name__ == '__main__':
    main(constrain=False)
    main(constrain=True)
\end{minted}
\caption{Full code for constrained dependency parsing problem
(Figure~\ref{fig:valency}).}
\end{listing}

\end{document}

%% file: fg.tex
\begin{tikzpicture}[
varon/.style={fill=gray},
vrbl/.style={draw,circle,fill=white,node distance=20pt and 55pt},
arclbl/.style 2 args={label={[fill=white,fill opacity=.9,text opacity=1]below:{\footnotesize #1$\rightarrow$#2}}},
fct/.style={fill},
treeedge/.style={color=treefactor,thick},
budedge/.style={color=budgetfactor,thick}
]

\pgfdeclarelayer{fg}
\pgfsetlayers{main,fg}

\begin{pgfonlayer}{fg}
\node[vrbl,arclbl={sleep}{the}]                         (st) at (0, 0) {};
\node[vrbl,arclbl={sleep}{clock}, right=of st]          (sc) {};
\node[vrbl,varon,arclbl={sleep}{around}, right=of sc]   (sa) {};
\node[vrbl,arclbl={the}{sleep}, below=of st]            (ts) {};
\node[vrbl,arclbl={the}{clock}, right=of ts]            (tc) {};
\node[vrbl,arclbl={the}{around}, right=of tc]           (ta) {};
\node[vrbl,arclbl={clock}{sleep}, below=of ts]          (cs) {};
\node[vrbl,varon,arclbl={clock}{the}, right=of cs]      (ct) {};
\node[vrbl,,arclbl={clock}{around}, right=of ct]        (ca) {};
\node[vrbl,arclbl={around}{sleep}, below=of cs]         (as) {};
\node[vrbl,arclbl={around}{the}, right=of as]           (at) {};
\node[vrbl,varon,arclbl={around}{clock}, right=of at]   (ac) {};
\end{pgfonlayer}

\node[fct, color=treefactor, left=40pt of ts] (tree) {};
\draw[treeedge] (tree) to[bend left=20] (st);
\draw[treeedge] (tree) to[bend left=30] (sc);
\draw[treeedge] (tree) to[bend left=40] (sa);
\draw[treeedge] (tree) to[bend right=5] (ts);
\draw[treeedge] (tree) to[bend right=5] (tc);
\draw[treeedge] (tree) to[bend right=5] (ta);
\draw[treeedge] (tree) to[bend right=10] (cs);
\draw[treeedge] (tree) to[bend right=15] (ct);
\draw[treeedge] (tree) to[bend right=18] (ca);
\draw[treeedge] (tree) to[bend right] (as);
\draw[treeedge] (tree) to[bend right=20] (at);
\draw[treeedge] (tree) to[bend right=28] (ac);

\node[fct, color=budgetfactor, above right=5pt and 25pt of st] (budget1) {};
\draw[budedge] (budget1) to (st);
\draw[budedge] (budget1) to (sc);
\draw[budedge] (budget1) to (sa);

\node[fct, color=budgetfactor, above right=5pt and 25pt of ts] (budget2) {};
\draw[budedge] (budget2) to (ts);
\draw[budedge] (budget2) to (tc);
\draw[budedge] (budget2) to (ta);

\node[fct, color=budgetfactor, above right=5pt and 25pt of cs] (budget3) {};
\draw[budedge] (budget3) to (cs);
\draw[budedge] (budget3) to (ct);
\draw[budedge] (budget3) to (ca);

\node[fct, color=budgetfactor, above right=5pt and 25pt of as] (budget4) {};
\draw[budedge] (budget4) to (as);
\draw[budedge] (budget4) to (at);
\draw[budedge] (budget4) to (ac);

\node[anchor=base,above left=25pt and 10pt of st] (sleep) {sleep};
\node[above=15pt of sleep] (rtn) {};
\node[anchor=base,right=0pt of sleep] (the) {the};
\node[anchor=base,right=0pt of the] (clock) {clock};
\node[anchor=base,right=0pt of clock] (around) {around};

\node[fct,color=budgetfactor, above left=45pt and 20pt of sa,
label=right:{\scriptsize \textsf{BUDGET} factor}] () {};
\node[fct,color=treefactor, above left=30pt and 20pt of sa,
label=right:{\scriptsize \textsf{TREE} factor}] () {};

\path[-triangle 45]
(rtn.south) edge (sleep.north)
(sleep.north) edge[bend left=30] (around.north)
(around.north) edge[bend right=20] (clock.north)
(clock.north) edge[bend right=20] (the.north);

\end{tikzpicture}

%% file: matchfg.tex
\begin{tikzpicture}[
varon/.style={fill=gray},
vrbl/.style={draw,circle,fill=white,node distance=10pt and 10pt},
arclbl/.style 2 args={label={[fill=white,fill opacity=.9,text opacity=1]below:{\footnotesize #1$\rightarrow$#2}}},
fct/.style={inner sep=2pt, fill},
treeedge/.style={color=treefactor,thick},
budedge/.style={color=budgetfactor,thick}
]

\pgfdeclarelayer{fg}
\pgfsetlayers{main,fg}

\begin{pgfonlayer}{fg}
\node[vrbl,]              (aa) at (0, 0) {};
\node[vrbl, right=of aa]  (ab) {};
\node[vrbl, right=of ab, varon]  (ac) {};
\node[vrbl, below=of aa, varon]  (ba) {};
\node[vrbl, right=of ba]  (bb) {};
\node[vrbl, right=of bb]  (bc) {};
\node[vrbl, below=of ba]  (ca) {};
\node[vrbl, right=of ca, varon]  (cb) {};
\node[vrbl, right=of cb]  (cc) {};

\node[vrbl,right=30pt of ac]   (Xaa) {};
\node[vrbl, right=of Xaa]      (Xab) {};
\node[vrbl, right=of Xab,varon]      (Xac) {};
\node[vrbl, below=of Xaa,varon]      (Xba) {};
\node[vrbl, right=of Xba]      (Xbb) {};
\node[vrbl, right=of Xbb]      (Xbc) {};
\node[vrbl, below=of Xba]      (Xca) {};
\node[vrbl, right=of Xca,varon]      (Xcb) {};
\node[vrbl, right=of Xcb]      (Xcc) {};

\end{pgfonlayer}

\node[fct, color=treefactor, above left=12pt of ba] (match) {};
\draw[treeedge] (match) to[bend right=10] (aa);
\draw[treeedge] (match) to[bend right=10] (ab);
\draw[treeedge] (match) to[bend right=10] (ac);
\draw[treeedge] (match) to[bend right=5] (ba);
\draw[treeedge] (match) to[bend right=5] (bb);
\draw[treeedge] (match) to[bend right=5] (bc);
\draw[treeedge] (match) to[bend right=20] (ca);
\draw[treeedge] (match) to[bend right=20] (cb);
\draw[treeedge] (match) to[bend right=20] (cc);

\node[fct, color=budgetfactor, above right=1pt and 2pt of Xaa] (xor1) {};
\draw[budedge] (xor1) to (Xaa);
\draw[budedge] (xor1) to (Xab);
\draw[budedge] (xor1) to (Xac);

\node[fct, color=budgetfactor, above right=1pt and 2pt of Xba] (xor2) {};
\draw[budedge] (xor2) to (Xba);
\draw[budedge] (xor2) to (Xbb);
\draw[budedge] (xor2) to (Xbc);

\node[fct, color=budgetfactor, above right=1pt and 2pt of Xca] (xor3) {};
\draw[budedge] (xor3) to (Xca);
\draw[budedge] (xor3) to (Xcb);
\draw[budedge] (xor3) to (Xcc);

\node[fct, color=budgetfactor, above right=7pt and 1pt of Xca] (xor4) {};
\draw[budedge] (xor4) to (Xca);
\draw[budedge] (xor4) to (Xba);
\draw[budedge] (xor4) to (Xaa);

\node[fct, color=budgetfactor, above right=7pt and 1pt of Xcb] (xor5) {};
\draw[budedge] (xor5) to (Xcb);
\draw[budedge] (xor5) to (Xbb);
\draw[budedge] (xor5) to (Xab);

\node[fct, color=budgetfactor, above right=7pt and 1pt of Xcc] (xor6) {};
\draw[budedge] (xor6) to (Xcc);
\draw[budedge] (xor6) to (Xbc);
\draw[budedge] (xor6) to (Xac);

\end{tikzpicture}

%% file: narralgo_fwd.tex
\begin{algorithm}[t]
\small\setstretch{1.15}
\caption{\label{alg:ad3qp}ADMM for LP-SparseMAP}

 \begin{algorithmic}[1]
    \STATE \textbf{Input:} $\pr$ (scores), $T$ (max.\ iterations), $\gamma$ (ADMM
step size),
$\varepsilon_p, \varepsilon_d$ (primal and dual stopping criteria).
    \STATE \textbf{Output:} ($\mg, \p$) solving \eqnref{eqn:lpsmap}.
\vspace{.5ex}
\STATE \textbf{Initialization:} $\mu^{(0)}_i = \nicefrac{1}{\deg(i)},
\bm{\lambda}^{(0)}=\bm{0}$.
\FOR{$t = 1, \dots,T$}
\FORALL{$f \in \mathcal{F}$
\Comment{SparseMAP subproblem}}
\STATE $\widetilde{\pr}_{f, M} \leftarrow \frac{1}{\gamma+1}\left(\bm{D}_f \pr_{f,M} - \lbd^{(t -
1)}_f + \gamma \CCs_f \mg^{(t - 1)}\right)$
\STATE $\widetilde{\pr}_{f, N} \leftarrow \frac{1}{\gamma+1} \pr_{f, N}$
\STATE $\p_f^{(t)} \leftarrow
\displaystyle \argmin_{\p_f \in \simplex_f}~
\frac{1}{2} \| \widetilde{\pr}_{f, M} - \MMs_f \p_f \|^2 -
\DP{\widetilde{\pr_{f, N}}}{\NN_f \p_f}$ \label{line:ad3qp-smap}%
\ENDFOR
\STATE  $\mg^{(t)} \leftarrow \CCs^\top \MMs \p^{(t)}$
\Comment{agreement by local averaging}
\STATE $\lbd^{(t)} \leftarrow \lbd^{(t-1)} + \gamma \big(\CCs \mg^{(t)} - \MMs \p^{(t)} \big)$
\Comment{dual update}
\IF{$\| \mg^{(t)} - \mg^{(t-1)} \| < \varepsilon_d$ \texttt{\&}
$\|\CCs \mg^{(t)} - \MMs \p^{(t)} \| < \varepsilon_p$}
\STATE return \Comment{converged}
%\Return  \Comment{converged}
\ENDIF
\ENDFOR
\end{algorithmic}
\end{algorithm}

%% file: narralgo_bck.tex
\begin{algorithm}[t]
\small\setstretch{1.15}
\caption{\label{alg:ad3qp_backward}Backward pass for LP-SparseMAP}

\begin{algorithmic}[1]
    \STATE \textbf{Input:} $\bm{d}$ (the gradient of the loss \wrt $\mg$),
$T$ (the maximum number of iterations), $\varepsilon$ (stopping criterion).
    \STATE \textbf{Output:} $\bm{d}_M, \bm{d}_{N, f}$ (loss gradient
\wrt $\pr_M$ and $\pr_{N,f}$).
\vspace{.5ex}
\FOR{$t = 1, \dots,T$}
\FORALL{$f \in \mathcal{F}$}
\STATE $\bm{d}_f \leftarrow \CCs_f \bm{d}$;
\Comment{split $\bm{d}$ into copies for each factor}%
\label{line:backsplit}
\STATE $\bm{d}_{M, f} \leftarrow  \JJ_{M, f}^\top \bm{d}_f,
~\bm{d}_{N, f} \leftarrow  \JJ_{N, f}^\top \bm{d}_f$;
\Comment{local $\nabla$}%
\label{line:backjacob}
\ENDFOR
\STATE $\bm{d}_M \leftarrow \sum_f \CCs_f^\top\bm{d}_f$.
\Comment{local averaging}%
\label{line:backmerge}
\IF{$\| \bm{d}_M - \bm{d} \| \leq \varepsilon$}
\STATE return $(\bm{d}_M, \bm{d}_{N, f})$. \Comment{converged}
\ELSE
\STATE $\bm{d} \leftarrow \bm{d}_M$
\ENDIF
\ENDFOR
\end{algorithmic}
\end{algorithm}